\documentclass{article}

  \usepackage[preprint]{neurips_2026}

\usepackage{algorithm}
\usepackage{bm}
\usepackage{amsmath,amsthm,mathtools,amssymb}
\usepackage{xurl}
\usepackage{algpseudocode}

\usepackage[utf8]{inputenc} 
\usepackage[T1]{fontenc}    
\usepackage[colorlinks, citecolor=blue,urlcolor=magenta]{hyperref}       
\usepackage{url}            
\usepackage{booktabs}       
\usepackage{amsfonts}       
\usepackage{nicefrac}       
\usepackage{microtype}      
\usepackage{xcolor}         

\usepackage{graphicx}
\usepackage{subcaption}
\usepackage{amsmath}
\usepackage[normalem]{ulem}
\usepackage{amssymb}
\usepackage{mathtools}
\usepackage{amsthm}
\usepackage{mathtools}
\usepackage{bbm}
\usepackage{thm-restate}
\usepackage{enumitem}
\usepackage{threeparttable}

\usepackage[capitalise,noabbrev]{cleveref}

\newtheorem{theorem}{Theorem}
\newtheorem{lemma}{Lemma}

\newtheorem{remark}{Remark}
\newcommand{\reftext}{\text{ref}}
\DeclareMathOperator{\DualGap}{\mathrm{DualGap}}
\newcommand{\uni}{\text{uni}}

\newcommand{\argmax}{\mathrm{argmax}}
\newcommand{\argmin}{\mathrm{argmin}}

\title{Inference-Time Nash Alignment}

\author{%
  Hadi Hosseini \\
  Penn State University, USA\\
 \texttt{hadi@psu.edu} \\
   \And
   Debmalya Mandal \\
   University of Warwick, UK \\ \texttt{Debmalya.Mandal@warwick.ac.uk} \\
   \AND
   Duohan Zhang\thanks{Corresponding author.} \\
   Penn State University, USA \\
   \texttt{dqz5235@psu.edu} \\
}

\begin{document}

\maketitle

\begin{abstract}

Preference-based fine-tuning methods such as RLHF and DPO require substantial compute and large preference datasets. They also need direct access to the model parameters which are not provided by many state-of-the art models. Inference-time alignment offers a cost-effective alternative without updating model parameters. However, existing inference-time methods rely on a scalar reward model derived under a Bradley-Terry assumption, which cannot represent general preferences. Following recent work on fine-tuning with generalized preferences, in this work, we initiate the study of inference-time alignment under general preferences.
We formulate the problem as obtaining a Nash equilibrium of a two-player zero-sum game between policies.
We propose two algorithms: \emph{Best-of-Nash} (BoN) and \emph{Nash Mirror Descent} (NMD). 
We prove that both algorithms achieve a duality gap that matches the problem lower bound.
Empirically, we implement the two methods on three datasets, which shows that our methods substantially outperform the base policy, converging to the performance of the fine-tuned models. Moreover, our results show that NMD remains robust across the regularization parameter.

\end{abstract}

\section{Introduction}
Preference-based fine-tuning methods such as Reinforcement Learning with Human Feedback (RLHF)~\citep{christiano2017deep} and Direct Preference Optimization (DPO)~\citep{rafailov2023direct} have become quintessential for aligning Large Language Models (LLMs) with human preferences.
These methods have proven highly effective across several domains  including mathematical reasoning and finance. 
However, they require substantial effort in acquiring high-quality human data, along with considerable computational costs for training LLMs.
For example, fine-tuning LLMs with either Proximal Policy Optimization (PPO)~\citep{schulman2017proximal} or Group Relative Policy Optimization (GRPO)~\citep{shao2024deepseekmath} requires a large number of samples and often leads to instability in training.
Equally importantly, these approaches rely on white-box access to model parameters, whereas many state-of-the-art models (e.g., GPT-5.4 Thinking \citep{openai2026gpt5}, Gemini 3.1 Pro \citep{google2026gemini3}, and Claude Opus 4.7 \citep{anthropic2026claude}) are only accessible as black-box APIs.

To bridge this gap, inference-time alignment has recently drawn much attention.
Unlike preference-based fine-tuning, inference-time alignment modifies the generation process at run time without updating model parameters.
A popular line of work adopts a \textit{sample-and-rank} paradigm, of which Best-of-$N$ sampling \citep{stiennon2020learning, nakano2021webgpt, huang2025best} is the most widely used -- $N$ candidate responses are drawn for a given prompt and the one with the highest score under a reward model is returned.
While simple in its nature, Best-of-$N$ alignment is vulnerable to \textit{reward hacking}: as $N$ grows, the estimated reward of the selected response increases monotonically while its true task performance can degrade \citep{gao2023scaling,stroebl2024inference,chow2024inference}.
This phenomenon arises because reward model is at best an imperfect proxy for the true human preference distribution, a manifestation of Goodhart's law.

Furthermore, most inference-time methods exclusively assume that preferences can be modeled by a scalar reward function under the Bradley-Terry model~\citep{bradley1952rank}.
This assumption is restrictive: even when individual preferences are transitive, aggregated group-level preferences need not be~\citep{munos2024nash}.  
Generalized preference models, which directly specify a probability that one response is preferred to another, have recently been studied for fine-tuning in, e.g., Nash Learning from Human Feedback (NLHF)~\citep{munos2024nash}.
To the best of our knowledge, tackling generalized preferences remains an open challenge in inference-time methods.

In this work, we initiate the study of black-box inference-time alignment under generalized preferences. 
The preference model $\mathbb{P}^*$ takes two responses $y$ and $y'$ conditioned on a prompt $x$, and evaluates the score $\mathbb{P}^*(y \succ y'|x)$ as the probability that a randomly chosen human prefers response $y$ over $y'$ given prompt $x$.
Then we formulate inference-time alignment with generalized preferences as a two-player zero-sum game $\mathbb{P}^*(\pi \succ \pi'|x) := \mathbb{E}_{y \sim \pi, y' \sim \pi'}\mathbb{P}^*(y \succ y'|x)$ between row player's policy $\pi$ and column player's policy $\pi'$ under the true preference $\mathbb{P}^*$.  
However, we do not have access to $\mathbb{P}^*$, and instead, use an imperfect preference model $\widehat{\mathbb{P}}$ (e.g. one learned from preference data \citep{jiang2023llm, dong2024rlhf, munos2024nash}) to query at inference time.
Specifically, we ask the following question: 
\begin{quote}
\textit{Given a base policy $\pi_{\reftext}$ from which we can query responses and an imperfect preference model $\widehat{\mathbb{P}}$, can we design an efficient algorithm to approximate the Nash equilibrium at inference time?}
\end{quote}
\subsection{Our Contributions}
Prior work has explored Nash-based objectives in the fine-tuning setting~\citep{munos2024nash} and has separately identified reward hacking as a key vulnerability of inference-time methods like Best-of-$N$~\citep{stroebl2024inference, chow2024inference}. We connect these two threads by proposing new inference-time methods with general preference models that achieve optimal regret. 
On the technical front, we link the alignment problem to the policy coverage and preference model error, which we will define formally in \Cref{sec:prelim}.

In particular, our main contributions are the following.
\begin{enumerate}[leftmargin=2em]
    \item \textbf{Best-of-Nash (BoN)}: We propose Best-of-Nash Alignment method which computes a \textit{minimax} solution from $N$ samples according to an estimate of the preference matrix $\hat{\mathbb{P}}$. We show that the \emph{duality gap} of BoN alignment is at most $O(\varepsilon(x) \mathcal{C}_{\text{uni}}(x))$ on a given prompt $x$, where $\varepsilon(x)$ is the error in the preference oracle, and $\mathcal{C}_{\text{uni}}(x)$ is a measure of data coverage.
        \item \textbf{Nash Mirror Descent (NMD)}: 
    As a more computationally efficient alternative, we propose Nash Mirror Descent, a self-play algorithm that takes a KL-regularized mirror descent step iteratively. We show that NMD is faster to implement, and the upper bound on the duality gap of NMD is at most $O(\varepsilon(x) \mathcal{C}_{\text{uni}}(x))$, matching the exact Best-of-Nash bound.
    \item \textbf{Matching Lower Bound}: We then show that the upper bound of BoN and NMD is essentially optimal by constructing problem instances with lower bound at least $\Omega(\varepsilon(x) \mathcal{C}_{\text{uni}}(x))$. 
    \item \textbf{Experimental Evaluation}: We evaluate our proposed mechanisms BoN and NMD on three preference datasets: TLDR, HelpSteer2, and UltraFeedback.
    We show that both methods substantially improve over the SFT base model, and match the win-rate of the fine-tuned alternative on TLDR.
    We also show that the win-rates of NMD are robust across the regularization parameter, removing the need for hyperparameter tuning.
\end{enumerate}
\subsection{Related Work}

\textbf{Alignment with General Preferences.}
\cite{azar2024general} initiated the study of general preferences in LLM fine-tuning, and proposed an algorithm that maximizes the KL-regularized objective against a fixed policy.
A subsequent line of works \citep{munos2024nash, ye2024theoretical, calandriello2024human, rosset2024direct, wu2024self,zhang2024iterative,zhang2025improving,swamy2024minimaximalist,zhou2025extragradient} formulated the alignment problem as a two-player zero-sum game, and proposed fine-tuning algorithms to learn the Nash policy. \citet{maura2025jackpot} connected this line of works to social choice theory, showing that the Nash policy approximates maximal lottery outcomes. 
A related line of work studies alignment when annotators have heterogeneous preferences, which a single reward model fails to aggregate~\citep{halpern2026pairwise, shirali2026direct, golz2026distortion}.

\textbf{Inference-Time Alignment and Reward Hacking.}
Best-of-N sampling is a popular inference-time alignment approach, but it is vulnerable to reward hacking problem~\citep{skalse2022defining}, when an LLM exploits the learned reward model rather than the ground-truth reward. 
\citet{huang2025best} showed that Best-of-N incurs suboptimal regret relative to the problem's lower bound, and proposed a $\chi^2$-regularized algorithm that implements pessimism in the face of uncertainty to close this gap.
\citet{yu2026curiosity} took a different route, using lower confidence bounds on value estimates to mitigate reward hacking. \citet{gui2024bonbon} combined RLHF with Best-of-N sampling to improve fine-tuning, however, they are concerned with reward-based setting. Finally, there is inference-time method based on search ~\citep{khanov2024args, yao2023tree} and rejection sampling~\citep{chen2024pad, shi2024decoding}, but these approaches don't handle general preferences.

\textbf{Learning in Zero-sum Games.} The central ingredient of our approach is no-regret learning algorithm for solving zero-sum games. \citet{freund1999adaptive} showed that multiplicative weights update has no-regret guarantee. 
We use Mirror descent algorithm~\citep{nemirovski2004prox, nesterov2009primal, rakhlin2013optimization} for inference-time alignment.  
Our second algorithm is inspired by the optimistic variants of Mirror descent algorithm~\citep{rakhlin2013optimization}.
Mirror Prox algorithm in \cite{nemirovski2004prox} is an extragradient type of method and closely related to optimistic MD (with similar convergence rate).
It queries the gradient twice whereas we use the same gradient twice (predicted and current).

\section{Preliminary}
\label{sec:prelim}
Denote $\mathcal{X}$ as the prompt space and $\mathcal{Y}$ as the response space.
We begin with a base policy $\pi_{\reftext}: \mathcal{X} \to \Delta(\mathcal{Y})$, where $\pi_{\reftext}(y|x)$ is the probability that the base policy $\pi_{\reftext}$ generates a response $y$ given the prompt $x$.
We also assume that there exists an unknown true preference oracle $\mathbb{P}^*: \mathcal{X} \times \mathcal{Y} \times \mathcal{Y} \to [0,1]$, where $\mathbb{P}^*(y \succ y'|x)$ denotes the probability that the population prefers response $y$ to $y'$ given a prompt $x$.
We assume that the preference model is \textit{skew-symmetric}:
\begin{equation*}
    \mathbb{P}^*(y \succ y'|x) + \mathbb{P}^*(y' \succ y|x) = 1, \forall x, y, y',
\end{equation*}
which implies $\mathbb{P}^*(y \succ y|x) = 1/2$.
As a proxy, we have access to an imperfect preference oracle $\widehat{\mathbb{P}}: \mathcal{X} \times \mathcal{Y} \times \mathcal{Y} \to [0,1]$, also assumed skew-symmetric.
For a given prompt $x$, we measure the quality of the oracle $\widehat{\mathbb{P}}$ via the square error with respect to $\mathbb{P}^*$, where responses are drawn independently from the base policy $\pi_{\reftext}$:
\begin{equation*}
    \varepsilon^2(x) := \mathbb{E}_{y \sim \pi_{\reftext}(\cdot|x), y' \sim \pi_{\reftext}(\cdot|x)}[(\widehat{\mathbb{P}}(y \succ y' |x) - \mathbb{P}^*(y \succ y' |x))^2]. 
\end{equation*}

\paragraph{Nash Equilibrium and Duality Gap.}
We formulate the problem as a two-player zero-sum game.
Given row player's policy $\pi$ and column player's policy $\pi'$, we denote the expected win-rate as
\begin{equation*}
    \mathbb{P}^*(\pi \succ \pi'|x) := \mathbb{E}_{y \sim \pi(\cdot|x), y' \sim \pi'(\cdot|x)}[\mathbb{P}^*(y \succ y'|x)].
\end{equation*}
Here the row player aims to maximize the win-rate, and the column player aims to minimize the win-rate.
It is well-known that there exists a Nash Equilibrium (NE) of the game:
\begin{equation*}
    \pi_1^*, \pi_2^* := \argmax_{\pi_1}\argmin_{\pi_2}\mathbb{P}^*(\pi_1 \succ \pi_2|x).
\end{equation*}
We denote the Nash Equilibrium as $ \pi^* = \pi_1^* = \pi_2^*$ due to the skew-symmetric nature of $\mathbb{P}^*$.
To measure how close a given policy $\pi$ is to $\pi^*$, we use the \textit{duality gap}. 
It captures the regret of policy $\pi$ as the gap between its strongest adversary and its weakest dominance under the true preference model:
\begin{equation*}
    \DualGap(\pi) := \max_{\pi_1}\mathbb{P}^*(\pi_1 \succ \pi |x) - \min_{\pi_2}\mathbb{P}^*(\pi \succ \pi_2|x)
\end{equation*}
The duality gap is nonnegative and $\DualGap(\pi ) = 0$ if $\pi = \pi^*$.
Now given a reference policy $\pi_{\reftext}$, an imperfect preference oracle $\widehat{\mathbb{P}}$, and a prompt $x \in \mathcal{X}$, our goal is to generate a high-quality policy $\hat{\pi}$ with small duality gap:
\begin{equation*}
    \DualGap(\hat{\pi}) \leq \epsilon.
\end{equation*}
We say $\hat\pi$ is an $\epsilon$-approximate Nash policy.

\paragraph{Universal Coverage.}
Coverage plays a significant role in the analysis of inference-time alignment \citep{huang2025best}. 
It measures how much a policy, $\pi$ concentrates probability mass relative to a reference policy, $\pi_{\reftext}$, upweighting outcomes that are more likely under $\pi$.
We define the coverage of policy $\pi$ as the ratio of $\pi(y|x)$ over $\pi_{\reftext}(y|x)$, where $y$ is sampled from $\pi(\cdot|x)$:
\begin{equation*}
    \mathcal{C}^{\pi}(x) := \mathop{\mathbb{E}}_{y \sim \pi(\cdot|x)}\left[\frac{\pi(y|x)}{\pi_{\reftext}(y|x)}\right].
\end{equation*}
Coverage is closely related to the \textit{chi-square divergence}: a direct calculation gives
\begin{equation*}
\chi^2(\pi(\cdot|x), \pi_{\reftext}(\cdot|x))  = \mathcal{C}^{\pi}(x) - 1,
\end{equation*}
where the chi-square divergence is defined as $\chi^2(\pi(\cdot|x), \pi'(\cdot|x)) = \mathbb{E}_{y \sim \pi'(\cdot|x)}[(\frac{\pi(y|x)}{\pi'(y|x)} - 1)^2]$.
Thus, we have $\mathcal{C}^{\pi}(x) \geq 1$ and the equality holds if and only if $\pi(y|x) = \pi_{\reftext}(y|x) $ for all $y$.
Inspired by literature in offline learning in zero-sum games~\citep{cui2022offline,zhong2022pessimistic,zhang2023offline}, we also define the \textit{universal coverage} as the maximum coverage over any policy:
\begin{equation*}
    \mathcal{C}_{\uni}(x) := \max_{\pi}\mathop{\mathbb{E}}_{y \sim \pi(\cdot|x)}\left[\frac{\pi(y|x)}{\pi_{\reftext}(y|x)}\right].
\end{equation*}

Intuitively, $\mathcal{C}_{\text{uni}}(x)$ captures the difficulty of recovering a Nash policy from samples drawn under $\pi_{\reftext}$.
The two quantities $\varepsilon(x)$ and $\mathcal{C}_{\text{uni}}(x)$ are the fundamental difficulty of inference-time alignment: no algorithm can output a good Nash approximation when the preference oracle $\widehat{\mathbb{P}}$ has high error or when $\pi_{\reftext}$ poorly covers the response space.

\section{The Best-of-Nash Algorithm}
\label{sec:bon}

The inference-time Best-of-$N$ alignment method relies on a scalar reward function by drawing $N$ samples from policy $\pi_{\reftext}$ and selecting a single response by taking the $\argmax$ under a reward model.
However, it condenses the preference information into one number and could result in reward hacking.

We propose an alternative approach, namely Best-of-Nash (\Cref{alg:best_of_nash}), that retains the sample-and-rank structure, but replaces the $\argmax$ with an equilibrium computation solely based on preference data: instead of picking one response, we output a distribution over the $N$ samples that solves the Nash equilibrium of the empirical preference game.

Formally, given an input $x$, we draw $N$ candidate responses $\widehat{\mathcal{Y}}_N = (y_1, \dots, y_N) \sim \pi_{\reftext}(\cdot|x)$ i.i.d. 
We then construct a probability matrix by querying $\widehat{\mathbb{P}}$ on each ordered pair, i.e. $\widehat{\mathbb{P}}(y_i \succ y_j)$ for $1 \leq i \leq j \leq N$.
The Nash equilibrium of the resulting two-player zero-sum game can be computed by Linear Programming (LP) \citep{adler2013equivalence}.

\begin{algorithm}[H]
    \caption{Best-of-Nash (BoN) Alignment}
    \label{alg:best_of_nash}
    \begin{algorithmic}[1]
        \State \textbf{Input}: Prompt $x$, reference policy $\pi_{\reftext}$, preference oracle $\widehat{\mathbb{P}}$, sample size $N$.
        \State Draw $\widehat{\mathcal{Y}}_N = (y_1, \dots, y_N) \sim \pi_{\reftext}(\cdot|x)$ i.i.d.
        \State Query $\widehat{P}_{ij} \gets \widehat{\mathbb{P}}(y_i \succ y_j \mid x)$ for all $1 \leq i < j \leq N$, and set $\widehat{P}_{ji} \gets 1 - \widehat{P}_{ij}$, $\widehat{P}_{ii} \gets \tfrac{1}{2}$.
        \State Compute a Nash Equilibrium $\hat\pi$ by solving the linear program 
        \begin{equation*}
            \max_{\pi \in \Delta(\widehat{\mathcal{Y}}_N),\, v \in \mathbb{R}} v
            \quad \text{s.t.} \quad
            \sum_{j=1}^{N} \pi(y_j)\, \widehat{P}_{ji} \geq v \quad \forall\, i \in [N].
        \end{equation*}
        \State \textbf{Return} $\hat\pi$.
    \end{algorithmic}
\end{algorithm}

We provide the duality gap guarantee for BoN. 
The bound depends on the two fundamental parameters: the preference-oracle error $\varepsilon(x)$ and the universal coverage $\mathcal{C}_{\text{uni}}(x)$.

\begin{theorem}
    \label{thm:upper_bound_bon}
    For any prompt $x$, the policy $\hat{\pi}$ returned by \Cref{alg:best_of_nash} satisfies
    \begin{equation*}
        \DualGap(\hat{\pi}) \leq   3 \varepsilon(x) \mathcal{C}_{\mathrm{uni}}(x) 
    \end{equation*}
    when $N \geq 4\log(\frac{2}{\varepsilon(x)}) \cdot \mathcal{C}_{\mathrm{uni}}(x) $.
\end{theorem}

We provide the full proof below, and defer omitted lemmas to \Cref{appendix:bon}.

\begin{proof}
We omit the dependence on $x$ for cleanliness.
From now on we write $S := \mathcal{Y}_N$ as the candidate set, and we write $\hat{\pi}_S$ to denote the dependence on $S$.
By skew-symmetry of the zero-sum game, it suffices to upper-bound 
$\mathbb{P}^*(\tilde\pi_S \succ \hat\pi_S) - 1/2$, where 
$\tilde\pi_S := \delta_{y^{*}}$ with
$y^{*} \in \arg\max_y \mathbb{P}^*(y \succ \hat{\pi}_S \mid x)$.
The central difficulty is a support mismatch: $\hat\pi_S$ is supported on the 
$N$ sampled responses, while $\tilde\pi_S$ may place mass anywhere in 
$\mathcal{Y}$. 
To bridge this, we introduce $\pi_{R,S}$, the distribution 
induced by approximate rejection sampling 
\citep{block2023sample, huang2025best} of $\tilde\pi_S$ from $\pi_{\reftext}$.
    Specifically, we denote $\pi_{R, S}$ as the distribution induced by $  \text{RejectionSampling}_{N-1,M}(\frac{\tilde\pi_S}{\pi_{\reftext}}; \pi_{\reftext}, x)$ (\Cref{alg:rejection_sampling}) as an approximation to $\tilde\pi$.
    Then we decompose the probability that
    \begin{equation*}      \mathbb{P}^*(\tilde\pi_S \succ \hat{\pi}_S ) \leq \mathbb{P}^*(\pi_{R,S} \succ \hat{\pi}_S) + |\mathbb{P}^*(\tilde\pi_S \succ \hat{\pi}_S) - \mathbb{P}^*(\pi_{R,S }\succ \hat{\pi}_S)|.
    \end{equation*}
We bound the first term $\mathbb{P}^*(\pi_{R,S} \succ \hat{\pi}_S) \leq \frac{1}{2} + \varepsilon(x)\mathcal{C}_{\text{uni}}(x)$ by \Cref{lem:bon_pi_R_bound}.
Then we bound the second term 
\begin{align*}
    |\mathbb{P}^*(\tilde\pi_S \succ \hat{\pi}_S) - \mathbb{P}^*(\pi_{R,S} \succ \hat{\pi}_S)| 
    &\leq \sum_{y}|\tilde\pi_S(y) - \pi_{R,S}(y)|\sum_{y'}\hat\pi_S(y') \mathbb{P}^*(y \succ y') \\ 
    &\leq \sum_{y}|\tilde\pi_S(y) - \pi_{R,S}(y)| \\
    &\leq 2D_{\text{TV}}(\tilde{\pi}_S, \pi_{R,S}),
\end{align*}
where the total-variation distance is defined as $D_{\text{TV}}(\pi, \pi'):= \frac{1}{2}\sum_{y}|\pi(y) - \pi'(y)|$.
Thus, the second term is reduced to the TV distance between $\tilde{\pi}_S$ and $\pi_{R,S}$, and then can be bounded due to the fact that $\pi_{R,S}$ is an approximation to $\tilde{\pi}_S$.
By \Cref{lem:total_variance}, We have 
\begin{equation*}
    D_{\text{TV}}(\tilde{\pi}_S, \pi_{R,S}) \leq \tfrac{1}{4}\varepsilon^2(x)\mathcal{C}_{\text{uni}}(x) 
\end{equation*} 
by setting $M = \frac{N-1}{\log(4/\varepsilon^2(x))}$ and $ N \geq 4\log(\frac{2}{\varepsilon(x)}) \cdot   \mathcal{C}_{\uni}(x)$.

By aggregating the bounds for both terms, we derive the upper bound
\begin{align*}
    \DualGap(\hat{\pi}_S) &\leq 2\mathcal{\mathbb{P}^*}(\tilde{\pi}_S \succ \hat{\pi}_S) - 1 \\
    &\leq 2(\varepsilon(x) \mathcal{C}_{\text{uni}}(x) + \tfrac{1}{2}\varepsilon^2(x)\mathcal{C}_{\text{uni}}(x)) \\
    &\leq  3 \varepsilon(x) \mathcal{C}_{\text{uni}}(x).
\end{align*}
For any fixed $y$, linearity gives
\begin{equation*}
    \mathbb{P}^*(y \succ \hat{\pi}) = \mathbb{P}^*(y \succ \mathbb{E}_S[\hat{\pi}_S]) = \mathbb{E}_S[\mathbb{P}^*(y \succ \hat{\pi}_S)].
\end{equation*}
Finally, we take the expectation over the randomness of $S$:
\begin{align*}
    \DualGap &= 2 \max_{y}\mathbb{E}_S[\mathbb{P}^*(y \succ \hat{\pi}_S)] - 1 \\
    &\leq 2 \mathbb{E}_S[\max_y \mathbb{P}^*(y \succ \hat{\pi}_S)] - 1 \\
    &= \mathbb{E}_S[\DualGap(\hat{\pi}_S)] \\
    &\leq 3 \varepsilon(x) \mathcal{C}_{\text{uni}}(x)
\end{align*}
and the proof is complete.
\end{proof}

\begin{remark}
    In \Cref{thm:upper_bound_bon} we assume that the LP solution is exact.
    When \Cref{alg:best_of_nash} returns an $\epsilon_{\text{LP}}$-approximation equilibrium, then we have $\DualGap \leq  3 \varepsilon(x) \mathcal{C}_{\text{uni}}(x) + 2 \epsilon_{\text{LP}}$.
\end{remark}

\section{The Nash Mirror Descent Algorithm}
\label{sec:nmd}
Best-of-Nash requires solving a linear programming with post-query time $O(N^{3.5} \log(1/\epsilon))$ via interior-point methods.
In this section, we propose \emph{Nash Mirror Descent} (\Cref{alg:nmd}), a self-play algorithm that achieves the same duality gap bound while replacing the LP with a sequence of closed-form updates.

NMD is inspired by \citet{rakhlin2013optimization} for solving zero-sum games.
The algorithm maintains two coupled policies: $\pi_t$ and $\pi'_t$, both supported on the $N$ sampled responses.
At each iteration, $\pi_{t+1}'$ takes a 
mirror-descent step from $\pi_t$ while staying close to $\pi'_t$:
\begin{equation*}
    \pi'_{t+1} = \argmax_{\pi \in \Delta(\mathcal{Y}_N)} \widehat{\mathbb{P}}(\pi \succ \pi_t |x) - \beta \cdot \text{KL}(\pi, \pi'_t),
\end{equation*}
and $\pi_{t+1}$ takes the same mirror descent step but stays close to $\pi'_{t+1}$:
\begin{equation*}
    \pi_{t+1} = \argmax_{\pi \in \Delta(\mathcal{Y}_N)} \widehat{\mathbb{P}}(\pi \succ \pi_t |x) - \beta \cdot \text{KL}(\pi, \pi'_{t+1}).
\end{equation*}
Here $\pi'_{t+1}$ aims to maximize the (estimated) probability that it wins against policy $\pi_t$, with an KL regularization term ensuring staying close to $\pi'_t$ (KL divergence is defined as $\text{KL}(\pi, \pi'):= \sum_{y}\pi(y)\log(\frac{\pi(y)}{\pi'(y)})$).
Both updates admit closed-form solutions.
Denote $\hat{r}_t(y) := \widehat{\mathbb{P}}(y \succ \pi_t |x) = \mathbb{E}_{y' \sim \pi_t}[\widehat{\mathbb{P}}(y \succ y' |x)], \forall y \in \mathcal{Y}_N$, we have

\begin{equation*}
    \pi'_{t + 1} = \argmax_{\pi \in \Delta(\mathcal{Y}_N)} \sum_{i}\hat{r}_{t}(y_i)\pi(y_i) - \beta \cdot \sum_{i}\pi(y_i) \log(\frac{\pi(y_i)}{\pi'_{t}(y_i)}),
\end{equation*}
and the solution is
\begin{equation*}
    \pi'_{t+1}(y_i) = \frac{\pi'_{t}(y_i) \exp(\hat{r}_{t}(y_i)/\beta)}{\sum_{j}\pi'_{t}(y_j)\exp(\hat{r}_{t}(y_j)/\beta)}.
\end{equation*}

\begin{algorithm}[t]
\caption{Nash Mirror Descent Alignment}
\label{alg:nmd}
    \begin{algorithmic}[1]
        \State \textbf{Input}: Prompt $x$, reference policy $\pi_{\reftext}$, preference oracle $\hat{\mathbb{P}}$, sample size $N$, regularization parameter $\beta$.
        \State Sample $N$ data $\mathcal{Y}_N = \{ y_1, y_2, \ldots, y_N \}$ i.i.d. from $\pi_{\reftext}(\cdot | x)$.
        \State Initialize $\pi'_1$ and $\pi_1$  as the uniform distribution on $\mathcal{Y}_N$. 
        \For {$t = 1, \ldots, T-1$ }
        \State \quad Calculate $\hat{r}_t(y) = \widehat{\mathbb{P}}(y \succ \pi_t |x) = \mathbb{E}_{y' \sim \pi_t}[\widehat{\mathbb{P}}(y \succ y' |x)], \forall y \in \mathcal{Y}_N$.
        \State \quad Calculate $\pi'_{t+1} = \argmax_{\pi \in \Delta(\mathcal{Y}_N)} \langle \pi, \hat{r}_{t} \rangle - \beta \cdot \text{KL}(\pi || \pi'_t)$.
        \State \quad Calculate $\pi_{t+1} = \argmax_{\pi \in \Delta(\mathcal{Y}_N)} \langle \pi, \hat{r}_{t} \rangle - \beta \cdot \text{KL}(\pi || \pi'_{t+1})$.
        \EndFor
        \State Calculate $\hat{\pi}(y|x) = \frac{1}{T}\sum_{t = 1}^{T}\pi_t(y|x), \forall y \in \mathcal{Y}_N$.
        \State \textbf{Return} $\hat{\pi}$.
    \end{algorithmic}
\end{algorithm}

\begin{restatable}{theorem}{thmnmd}
    \label{thm:upper_bound_omd}
    For any prompt $x$, by setting $\beta = 2$, $T = \lceil \frac{2\log N + 1/2}{ \varepsilon(x)\mathcal{C}_{\uni}(x)} \rceil$, the policy $\hat\pi$ returned by \Cref{alg:nmd} has the duality gap
    \begin{equation*}
        \DualGap(\hat{\pi}) \leq  5\varepsilon(x)\mathcal{C}_{\mathrm{uni}}(x). 
    \end{equation*}
    when $N \geq 4\log(\frac{2}{\varepsilon(x)}) \cdot \mathcal{C}_{\mathrm{uni}}(x) $.
\end{restatable}

We provide a proof sketch below, and defer the whole proof to \Cref{appendix:nmd}.
\begin{proof}[Proof Sketch]
The structure parallels the proof of Theorem~\ref{thm:upper_bound_bon}. 
By the skew-symmetry of the zero-sum game, it suffices to upper-bound 
$\mathbb{P}^*(\tilde\pi \succ \hat\pi) - 1/2$, where 
$\tilde\pi := \delta_{y^{*}}$ with
$y^{*} \in \arg\max_y \mathbb{P}^*(y \succ \hat{\pi} \mid x)$.
Introducing the rejection-sampling approximation $\pi_R$, we decompose
\begin{equation*}
\mathbb{P}^*(\tilde\pi \succ \hat\pi) 
\leq \mathbb{P}^*(\pi_R \succ \hat\pi) 
   + \bigl|\mathbb{P}^*(\tilde\pi \succ \hat\pi) - \mathbb{P}^*(\pi_R \succ \hat\pi)\bigr|.
\end{equation*}
The second term is controlled by the property of approximate rejection sampling, similarly to that in Theorem~\ref{thm:upper_bound_bon}.
For the first term, we further decompose
\begin{equation*}
\mathbb{P}^*(\pi_R \succ \hat\pi) 
\leq \widehat{\mathbb{P}}(\pi_R \succ \hat\pi) 
   + \bigl|\widehat{\mathbb{P}}(\pi_R \succ \hat\pi) - \mathbb{P}^*(\pi_R \succ \hat\pi)\bigr|,
\end{equation*}
bounding the empirical term via the cumulative-regret guarantee of Nash Mirror Descent on the game $\widehat{\mathbb{P}}$, and the transfer term by $\mathcal{C}_{\text{uni}}(x) \varepsilon(x)$. 
\end{proof}

\begin{remark}
    Our theory prescribes a specific $\beta = 2$.
    In practice, our experiments (\Cref{sec:exp}) show that the performance of NMD is empirically robust to the choice of $\beta$. 
\end{remark}
 
\begin{remark}[Comparison of BoN and NMD]
\label{rem:time}
    \Cref{thm:upper_bound_bon} and \Cref{thm:upper_bound_omd} give the same duality gap bound $O(\varepsilon(x) \mathcal{C}_{\uni}(x))$, and both algorithms share the same query complexity: $O(N)$ samples from $\pi_{\reftext}$ and $O(N^2)$ preference queries to $\widehat{\mathbb{P}}$.
    The two algorithms differ in post-query complexity: BoN solves an LP in $O(N^{3.5}\log(1/\epsilon))$ time, while NMD runs $T = O(\frac{\log(N)}{\varepsilon(x)\mathcal{C}_{\uni}(x)})$ updates of $O(N^2)$ each, for a total of $O(\frac{N^2 \log N}{\varepsilon(x)\mathcal{C}_{\uni}(x)})$.
\end{remark}

\section{Lower Bound}
\label{sec:lower}

The upper bounds in Theorems~\ref{thm:upper_bound_bon} and~\ref{thm:upper_bound_omd} show that 
both BoN and NMD achieve duality gap $O(\varepsilon(x) \mathcal{C}_{\text{uni}}(x))$ with an appropriate choice of parameters. 
A natural question is whether this rate is optimal.
In this section, we answer this affirmatively by constructing problem instances on which any inference-time algorithm must incur duality gap $\Omega(\varepsilon(x) \mathcal{C}_{\text{uni}}(x))$.

\begin{restatable}{theorem}{thmlower}
   \label{thm:lower_bound}
    Given a prompt $x$ and $K$ responses $\{ y_1,\ldots, y_K\}$, let $\pi_{\reftext}(y_i|x) = 1/K$ for all $i \in [K]$. 
    For any alignment algorithm $\mathcal{A}$ and any $\varepsilon_0 \in (0, \sqrt{2}/(3K)]$, there exist preference oracles $\mathbb{P}^*$ and $\widehat{ \mathbb{P}}$ with $\varepsilon(x) = \varepsilon_0$ such that:
    \begin{equation*}
        \DualGap(\mathcal{A}(\widehat{\mathbb{P}})) \geq \frac{\varepsilon_0 \cdot \mathcal{C}_{\mathrm{uni}}(x)}{2\sqrt{2}}.
    \end{equation*} 
\end{restatable}
\begin{proof}[Proof Sketch]
We set the construction: $\widehat{ \mathbb{P}}(y_1 \succ y_j) = \tfrac{1}{2} + \delta_0, \forall j \geq 2$, and $\widehat {\mathbb{P}}(y_j \succ y_k) = \tfrac{1}{2}, \forall j,k \geq 2$ for a small value $\delta_0 = \varepsilon_0 \mathcal{C}_{\text{uni}}(x) / (2\sqrt{2})$.
This makes $y_1$ the dominant strategy under $\widehat{\mathbb{P}}$.
We assume that $p := \hat{\pi}(y_1)$.
Then we construct two real worlds $A$ and $B$.
In world $A$, $\mathbb{P}^*_A$ amplifies $y_1$'s margin over $y_2$, 
keeping $y_1$ dominant. In world $B$, $\mathbb{P}^*_B$ flips the 
$(y_1, y_2)$ entry so that $y_2$ narrowly beats $y_1$, making $y_2$ the 
new dominant strategy.
The duality gaps in the two worlds scale as 
$2(1-p)\delta_0$ and $2p\delta_0$ respectively, and the adversary picks 
the larger.
The algorithm minimizes the maximum at $p = 1/2$, leaving a 
duality gap of at least $\delta_0 = \varepsilon_0 \mathcal{C}_{\text{uni}}(x) / (2\sqrt{2})$.
\end{proof}
The full proof is relegated to \Cref{appendix:lower}.
\begin{remark}[Optimality of BoN and NMD]
    Combining \Cref{thm:lower_bound} with \Cref{thm:upper_bound_bon} and \Cref{thm:upper_bound_omd} yields a tight characterization: the optimal duality gap for inference-time alignment with general preferences is $\Theta(\varepsilon(x)\mathcal{C}_{\text{uni}}(x))$, achieved by both of our algorithms.
\end{remark}

\begin{remark}[Contrast with reward-based alignment.]
The duality gap $\Omega(\varepsilon(x) \mathcal{C}_{\text{uni}}(x))$ stands in sharp contrast to the optimal rate for the reward-based inference time alignment, which \citep{huang2025best} show is $\Theta(\varepsilon_{\text{RM}}(x) \sqrt{\mathcal{C}^{\pi^*}(x)})$.
    Here $\varepsilon_{\text{RM}}(x)$ is the reward model error, and $\pi^*$ is the comparator policy.
    Two structural differences are notable. 
    First, the coverage dependence is linear in $\mathcal{C}_{\text{uni}}(x)$ in our setting versus $\sqrt{\mathcal{C}^{\pi^*}(x)}$ in the reward setting. This implies that bad coverage of the base policy hurts more when preferences are general than when they are scalar. 
    Second, our relevant coverage quantity is the universal coverage $\mathcal{C}_{\text{uni}}(x)$, not single-policy coverage $\mathcal{C}^{\pi^*}$, reflecting that inference-time alignment with general preferences is fundamentally harder than its reward-based counterpart. 
\end{remark}

\section{Experiments}
\label{sec:exp}
We evaluate our proposed alignment algorithms,  Best-of-Nash and Nash Mirror Descent, on three datasets and study how their performance varies with sample size $N$, regularization parameter $\beta$, and the choice of base models and preference models.

\subsection{Setup}

\paragraph{Datasets and Models.}
We conduct experiments on three preference datasets that are commonly used for training and evaluating LLM alignment: \href{https://huggingface.co/datasets/CarperAI/openai_summarize_tldr}{TLDR} 
(text summarization), \href{https://huggingface.co/datasets/nvidia/HelpSteer2}{HelpSteer2} 
(general-purpose helpfulness)~\citep{wang2024helpsteer2}, and \href{https://huggingface.co/datasets/HuggingFaceH4/ultrafeedback_binarized}{UltraFeedback}~\citep{cui2023ultrafeedback} 
(instruction following). 
We sample $100$ prompts for evaluation in each dataset. 
In our experiments, we consider three supervised fine-tuned (SFT) models as base models: \href{https://huggingface.co/RLHFlow/LLaMA3-SFT}{LLaMA3-SFT (8B)}~\citep{dong2024rlhf}, \href{https://huggingface.co/mistralai/Mistral-7B-Instruct-v0.3}{Mistral-Instruct (7B)}, and \href{https://huggingface.co/google/gemma-2b-it}{Gemma-SFT (2B)}.
Two aligned models are selected as our estimate preference oracle,  $\hat{\mathbb{P}}$: \href{https://huggingface.co/RLHFlow/pair-preference-model-LLaMA3-8B}{LLaMA3-PM (8B)}~\citep{dong2024rlhf} and \href{https://huggingface.co/llm-blender/PairRM-hf}{PairRM (0.4B)}~\citep{jiang2023llm}.

\paragraph{Evaluation.}
Our headline metric is \emph{expected win rate} (EWR := win + draw/2).
Results in the main paper mainly use LLaMA3-SFT as the base model until otherwise stated.
We also compare results to 
\href{https://huggingface.co/RLHFlow/LLaMA3-iterative-DPO-final}{LLaMA3-DPO (8B)}, a fine-tuned version of LLaMA3-SFT, to illustrate the performance of our algorithm against a model fine-tuned on preference data.
We use an LLM as a judge~\citep{zheng2023judging} since alignment is inherently preference-based, and LLMs provide a scalable and consistent proxy for human evaluations of model outputs. Concretely, we compare each generated response against a reference answer from the dataset, with judgments produced by DeepSeek-V4-Flash.
To \textit{control positional} bias~\citep{zheng2023judging,wang2024large}, we query the judge twice per pair with orderings swapped, counting a win/loss only when both orderings agree and a draw otherwise. 

\subsection{Main Results.}

\begin{table*}[t]
\centering
\small
\setlength{\tabcolsep}{6pt}
\renewcommand{\arraystretch}{1.1}
\begin{tabular}{lccc}
\toprule
Dataset & Base SFT  & BoN & NMD \\
\midrule
TLDR      & 62.9\%  & 73.5\% & 73.1\% \\
HelpSteer2  & 44.1\%  & 68.8\% & 67.8\% \\
UltraFeedback & 23.9\%  & 48.8\% & 46.7\%\\
\bottomrule
\end{tabular}
\caption{Comparison of expected win-rate (EWR) across three datasets.
BoN and NMD sample $N = 64$ responses from LLaMA3-SFT and we use LLaMA3-PM as preference model; $\beta = 1$ for NMD. 
}
\label{tab:dataset}
\end{table*}
\footnotetext[1]{LLaMA3-DPO model was fine-tuned on HelpSteer2 and UltraFeedback, making direct comparison unfair on these datasets.}

\Cref{tab:dataset} presents our headline comparison: BoN and NMD against the SFT base policy across all three datasets.
Both BoN and NMD outperform the base SFT across all three datasets: both algorithms have an improvement of roughly $10\%$ in TLDR, and $\approx24\%$ in HelpSteer2 and  UltraFeedback. 
BoN and NMD have similar expected win-rates on all three datasets, consistent with the theoretical guarantees in \Cref{thm:upper_bound_bon} and \Cref{thm:upper_bound_omd}.

\paragraph{Comparison to DPO}
\Cref{fig:tldr_llama} shows that BoN and NMD match the expected win-rate of LLaMA3-DPO with a reasonable number of samples, without any parameter updates.
This suggests that careful inference-time alignment can substitute for fine-tuning when the preference oracle is sufficiently strong.

\paragraph{Sample size $N$ and regularization $\beta$.} 
\Cref{fig:tldr_bon_llama} shows BoN's expected win-rate on TLDR as $N$ increases. 
EWR rises monotonically from $62.9\%$ at $N = 1$ to $73.5\%$ at $N = 64$, approaching the DPO baseline. 
This confirms the predicted scaling: with more samples, BoN better approximates the Nash equilibrium.
\Cref{fig:tldr_omd_llama} shows that NMD's EWR is flat across $\beta$, thus removing the need for hyper-parameter tuning.

\begin{figure}[t]
    \centering
    \begin{subfigure}{0.48\textwidth}
        \centering
        \includegraphics[width=\textwidth]{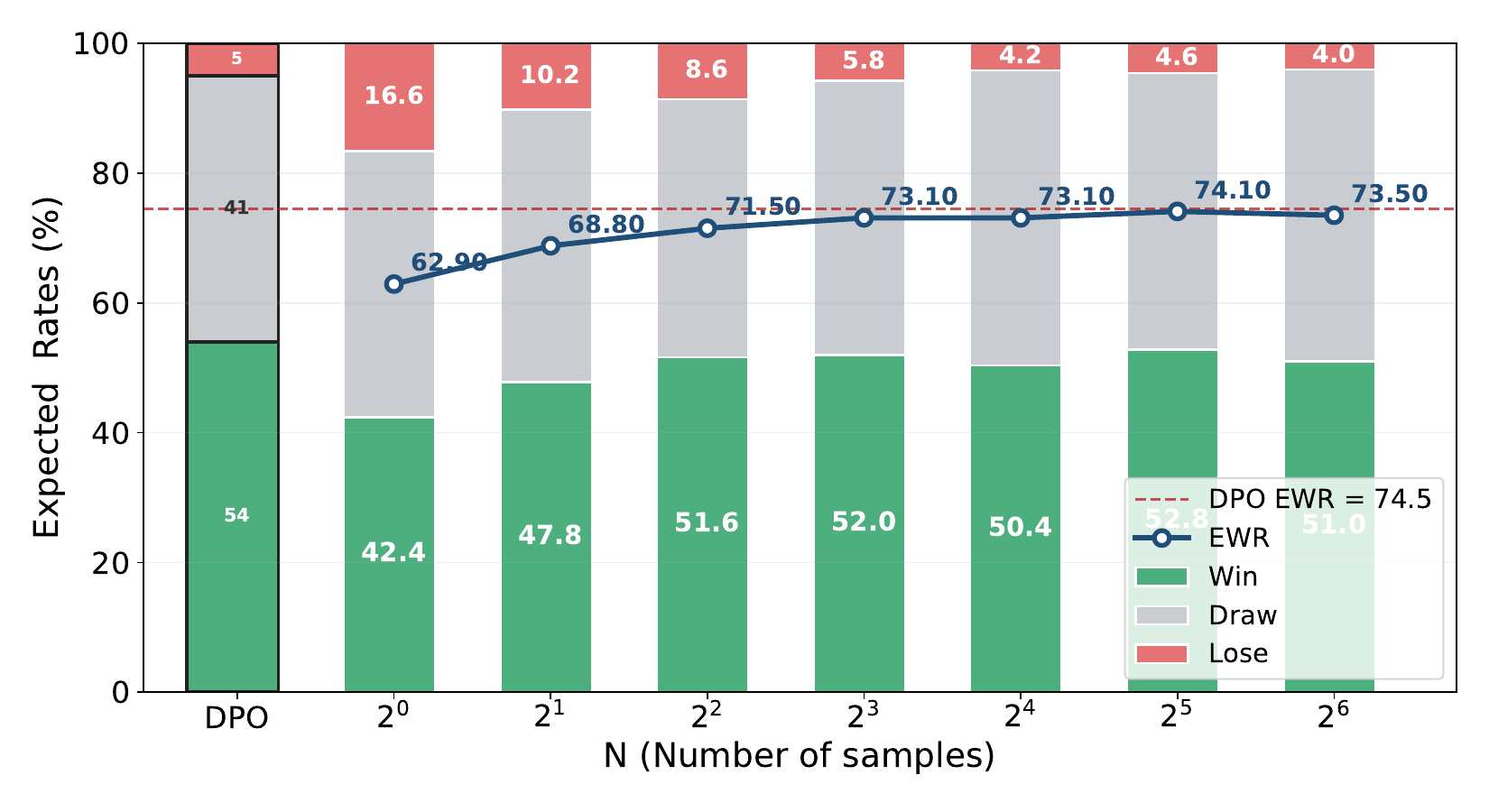}
        \caption{ BoN with varying sample size $N$.}
        \label{fig:tldr_bon_llama}
    \end{subfigure}
    \hfill 
    \begin{subfigure}{0.48\textwidth}
        \centering
        \includegraphics[width=\textwidth]{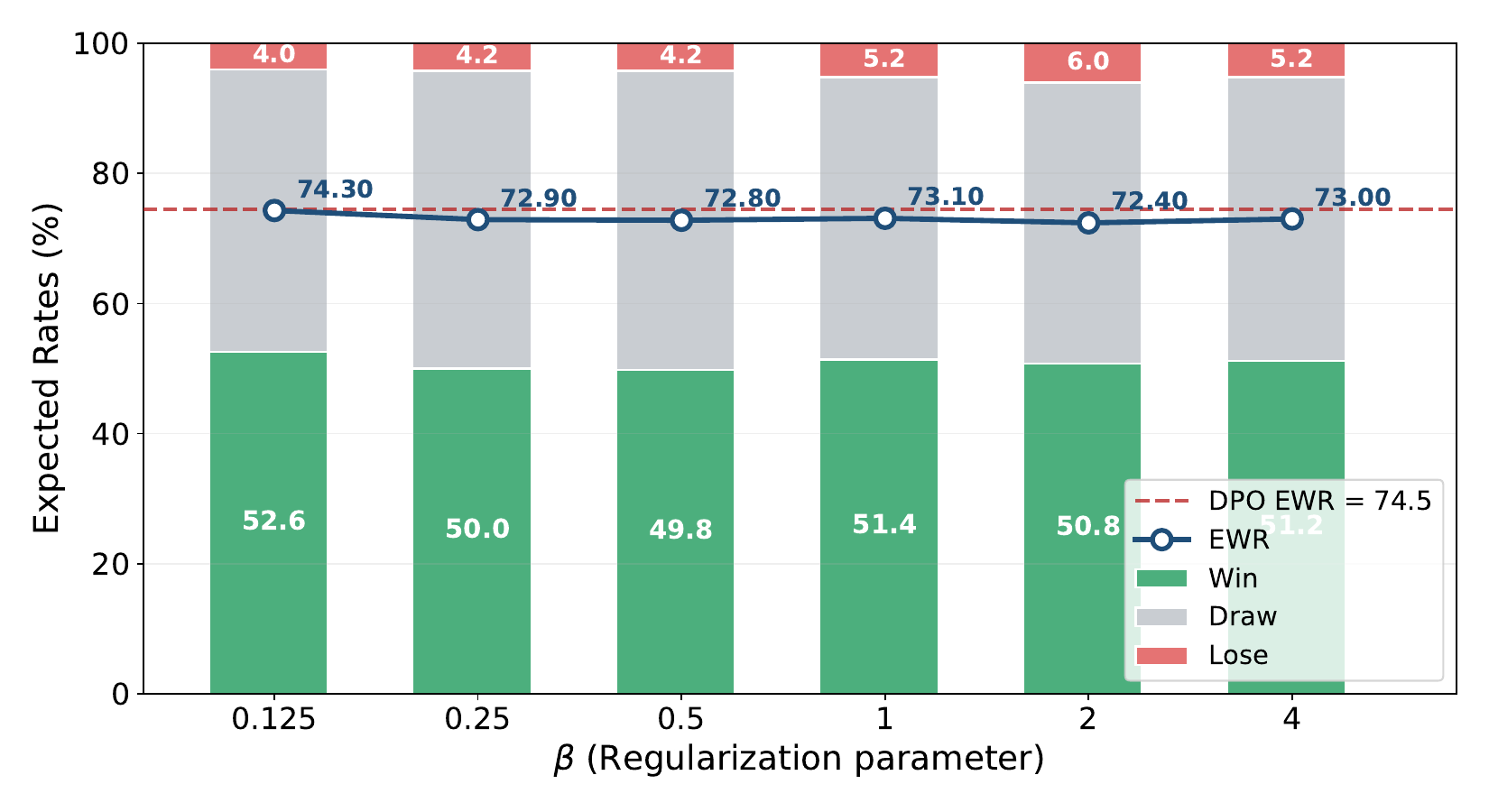}
        \caption{NMD with varying $\beta$ and fixed $N = 64$.}
        \label{fig:tldr_omd_llama}
        
    \end{subfigure}
    
    \caption{Win/Draw/Lose distribution and expected win-rate (Win + Draw/2) on the TLDR dataset. The dashed line indicates the baseline score for LLaMA3-DPO. We use LLaMA3-SFT as the reference policy and LLaMA3-PM as the preference oracle $\widehat{\mathbb{P}}$.}
    \label{fig:tldr_llama}
\end{figure}

\begin{figure}[t]
    \centering
    \begin{subfigure}{0.48\textwidth}
        \centering
        \includegraphics[width=\textwidth]{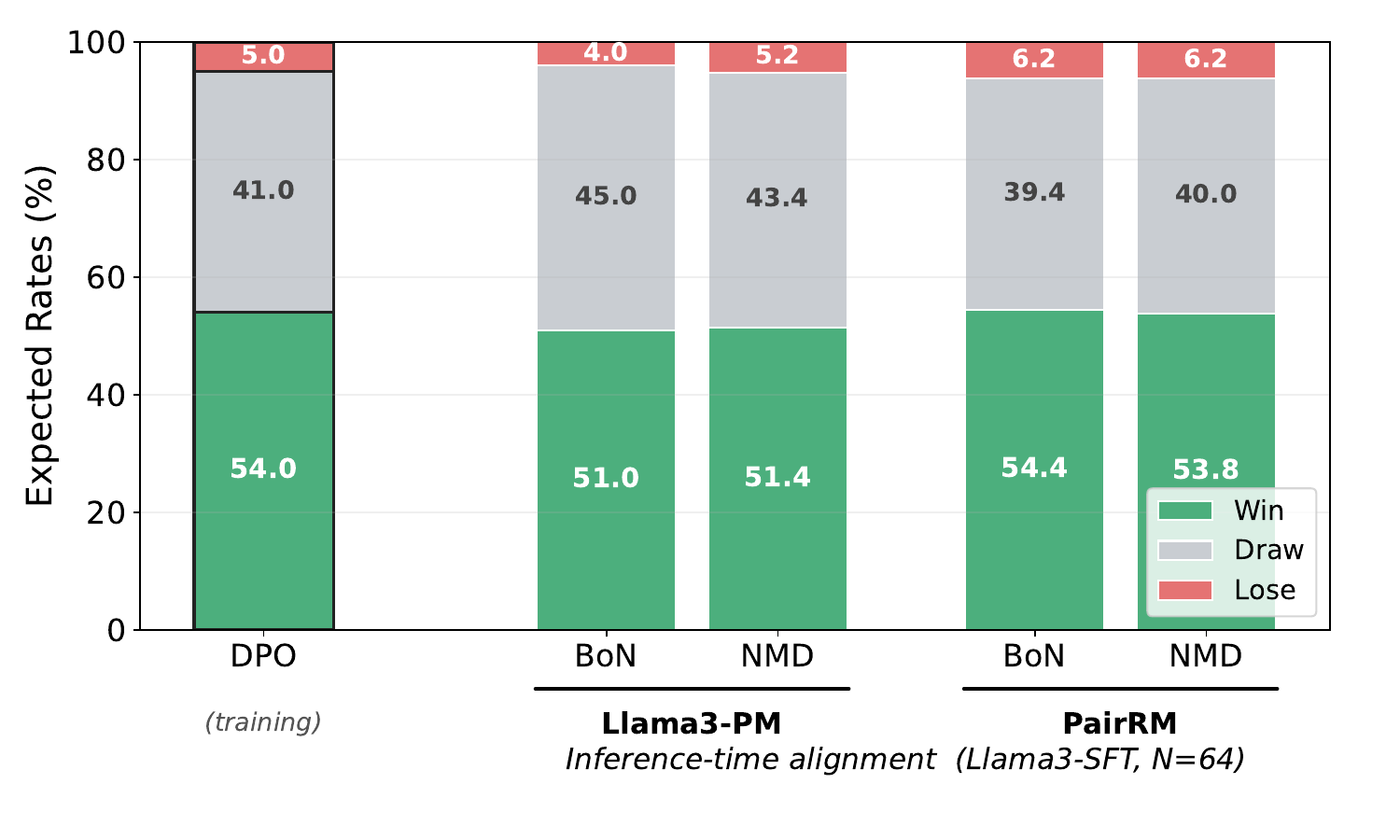}
    \end{subfigure}
    \hfill 
    \begin{subfigure}{0.48\textwidth}
        \centering
        \includegraphics[width=\textwidth]{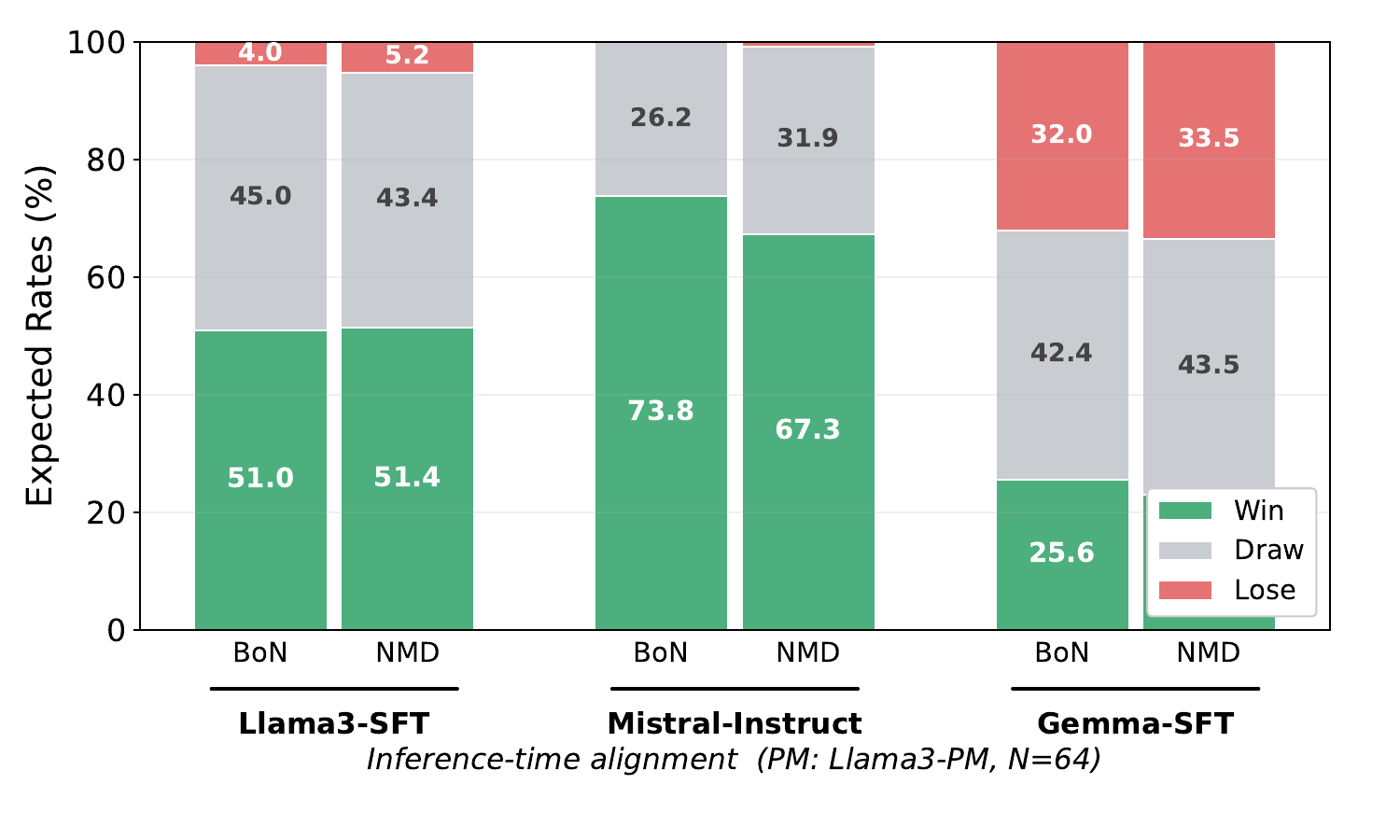}
        
    \end{subfigure}

    \caption{Win/Draw/Lose distribution on the TLDR dataset. 
    \textbf{Left}: we compare the Llama3-PM and PairRM for preference model, with the base model = LLama3-SFT.
    \textbf{Right}: we compare LLama3-SFT, Mistral-instruct, and Gemma-SFT for base model, with the preference model = LLama3-PM.}
    \label{fig:different_models}
    
\end{figure}

\paragraph{Preference oracle and base model.}
We show the effect of the preference model $\widehat{\mathbb{P}}$ and base model $\pi_{\reftext}$ in \Cref{fig:different_models}.
The left panel compares LLaMA3-PM (8B) and PairRM (0.4B) as the preference oracle with LLaMA3-SFT fixed as the base model: despite the 20 times difference in size, the two yield rather similar EWRs on TLDR for both BoN and NMD.
The right panel compares the three base models with LLaMA3-PM fixed as the preference oracle: Mistral-Instruct attains the highest EWR, followed by LLaMA3-SFT, with Gemma-SFT trailing by a clear margin.

\begin{figure}
    \centering
    \includegraphics[width=0.48\textwidth, height=1.6in]{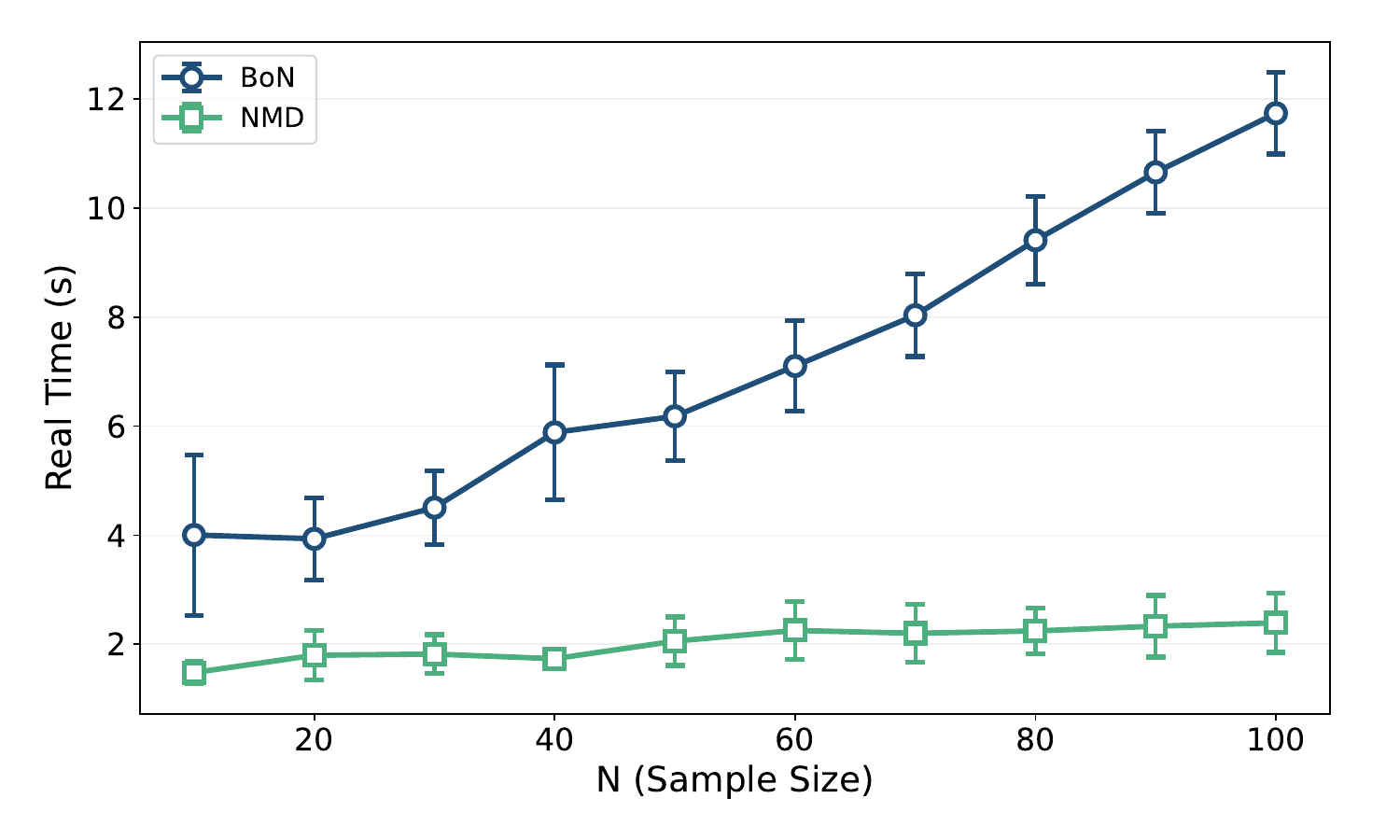}
    \caption{Post-query time comparison of BoN and NMD for varying size $N$. 
    }
    \label{fig:real_time}
\end{figure}

\paragraph{Post-query alignment run-time.}
As we show in \Cref{rem:time}, BoN and NMD share the same query complexity, but differ in post-query computation.
\Cref{fig:real_time} compares the post-query time of the two algorithms as $N$ grows: BoN scales rapidly with $N$, while NMD remains nearly flat.
This is consistent with our time complexity analysis.

Omitted experiment details and additional experiments are deferred to \Cref{sec:add_exp}, including comparing our methods against three baseline methods.

\section{Conclusion and Limitations}
We initiate the study of inference-time alignment under general preferences, formulating the problem as computing a Nash equilibrium of a two-player zero-sum game between policies under an imperfect preference oracle.
We propose two algorithms, Best-of-Nash (BoN) and Nash Mirror Descent (NMD), and prove that both algorithms achieve a duality gap of $O(\varepsilon(x) C_\text{uni}(x))$, which we show to be tight via a matching lower bound. 
The two algorithms differ in post-query computation: BoN solves a linear programming in $O(N^{3.5} \log(1/\epsilon))$ time, while NMD requires $O(\frac{N^2}{\varepsilon(x) C_\text{uni}(x)} \log(N))$ time.

Our work has several limitations.
First, our theoretical guarantees depend on the quality of preference models via $\varepsilon(x)$.
When $\widehat{\mathbb{P}}$ is a poor proxy for the true preference $\mathbb{P}^*$, for example, on prompts that fall outside the distribution on which $\widehat{\mathbb{P}}$ is trained, the methods may inherit the biases of $\widehat{\mathbb{P}}$.
Second, inference-time alignment can only re-weight responses that are reachable under $\pi_{\reftext}$. 
When the base policy assigns negligible probability to high-quality responses, neither BoN nor NMD can recover them, and fine-tuning remains necessary.
Lastly, our analysis treats each prompt independently and provides per-prompt guarantees, thereby neglecting shared structure across prompts such as similar tasks or styles.

We end this section with a few future directions.
Both BoN and NMD require $O(N^2)$ pairwise queries to $\widehat{\mathbb{P}}$, which can be a huge query cost; reducing this cost via active selection of pairs is an interesting direction.
Another direction is to combine our inference-time theory with fine-tuning with general preferences.
Finally, exploiting structure across prompts could improve sample-efficiency beyond the per-prompt rates we establish here.

\section*{Acknowledgments}
We thank the anonymous reviewers for their comments and constructive feedback.
 HH acknowledges support from the National Science Foundation, NSF Awards IIS-2144413 and IIS-2107173.


\bibliographystyle{plainnat}
\bibliography{ref}

\begin{thebibliography}{51}
\providecommand{\natexlab}[1]{#1}
\providecommand{\url}[1]{\texttt{#1}}
\expandafter\ifx\csname urlstyle\endcsname\relax
  \providecommand{\doi}[1]{doi: #1}\else
  \providecommand{\doi}{doi: \begingroup \urlstyle{rm}\Url}\fi

\bibitem[Adler(2013)]{adler2013equivalence}
Ilan Adler.
\newblock The equivalence of linear programs and zero-sum games.
\newblock \emph{International Journal of Game Theory}, 42\penalty0 (1):\penalty0 165--177, 2013.

\bibitem[{Anthropic}(2026)]{anthropic2026claude}
{Anthropic}.
\newblock {Introducing Claude Opus 4.7}.
\newblock \url{https://www.anthropic.com/news/claude-opus-4-7}, 2026.

\bibitem[Azar et~al.(2024)Azar, Guo, Piot, Munos, Rowland, Valko, and Calandriello]{azar2024general}
Mohammad~Gheshlaghi Azar, Zhaohan~Daniel Guo, Bilal Piot, Remi Munos, Mark Rowland, Michal Valko, and Daniele Calandriello.
\newblock A general theoretical paradigm to understand learning from human preferences.
\newblock In \emph{International Conference on Artificial Intelligence and Statistics}, pages 4447--4455. PMLR, 2024.

\bibitem[Block and Polyanskiy(2023)]{block2023sample}
Adam Block and Yury Polyanskiy.
\newblock The sample complexity of approximate rejection sampling with applications to smoothed online learning.
\newblock In \emph{The Thirty Sixth Annual Conference on Learning Theory}, pages 228--273. PMLR, 2023.

\bibitem[Bradley and Terry(1952)]{bradley1952rank}
Ralph~Allan Bradley and Milton~E Terry.
\newblock Rank analysis of incomplete block designs: I. the method of paired comparisons.
\newblock \emph{Biometrika}, 39\penalty0 (3/4):\penalty0 324--345, 1952.

\bibitem[Calandriello et~al.(2024)Calandriello, Guo, Munos, Rowland, Tang, Pires, Richemond, Lan, Valko, Liu, et~al.]{calandriello2024human}
Daniele Calandriello, Daniel Guo, Remi Munos, Mark Rowland, Yunhao Tang, Bernardo~Avila Pires, Pierre~Harvey Richemond, Charline~Le Lan, Michal Valko, Tianqi Liu, et~al.
\newblock Human alignment of large language models through online preference optimisation.
\newblock \emph{arXiv preprint arXiv:2403.08635}, 2024.

\bibitem[Chen et~al.(2024)Chen, Zhang, Luo, Chai, and Liu]{chen2024pad}
Ruizhe Chen, Xiaotian Zhang, Meng Luo, Wenhao Chai, and Zuozhu Liu.
\newblock Pad: Personalized alignment of llms at decoding-time.
\newblock \emph{arXiv preprint arXiv:2410.04070}, 2024.

\bibitem[Chow et~al.(2024)Chow, Tennenholtz, Gur, Zhuang, Dai, Thiagarajan, Boutilier, Agarwal, Kumar, and Faust]{chow2024inference}
Yinlam Chow, Guy Tennenholtz, Izzeddin Gur, Vincent Zhuang, Bo~Dai, Sridhar Thiagarajan, Craig Boutilier, Rishabh Agarwal, Aviral Kumar, and Aleksandra Faust.
\newblock Inference-aware fine-tuning for best-of-n sampling in large language models.
\newblock \emph{arXiv preprint arXiv:2412.15287}, 2024.

\bibitem[Christiano et~al.(2017)Christiano, Leike, Brown, Martic, Legg, and Amodei]{christiano2017deep}
Paul~F Christiano, Jan Leike, Tom Brown, Miljan Martic, Shane Legg, and Dario Amodei.
\newblock Deep reinforcement learning from human preferences.
\newblock \emph{Advances in neural information processing systems}, 30, 2017.

\bibitem[Cui et~al.(2023)Cui, Yuan, Ding, Yao, Zhu, Ni, Xie, Liu, and Sun]{cui2023ultrafeedback}
Ganqu Cui, Lifan Yuan, Ning Ding, Guanming Yao, Wei Zhu, Yuan Ni, Guotong Xie, Zhiyuan Liu, and Maosong Sun.
\newblock Ultrafeedback: Boosting language models with high-quality feedback, 2023.

\bibitem[Cui and Du(2022)]{cui2022offline}
Qiwen Cui and Simon~S Du.
\newblock When are offline two-player zero-sum markov games solvable?
\newblock \emph{Advances in Neural Information Processing Systems}, 35:\penalty0 25779--25791, 2022.

\bibitem[Dong et~al.(2024)Dong, Xiong, Pang, Wang, Zhao, Zhou, Jiang, Sahoo, Xiong, and Zhang]{dong2024rlhf}
Hanze Dong, Wei Xiong, Bo~Pang, Haoxiang Wang, Han Zhao, Yingbo Zhou, Nan Jiang, Doyen Sahoo, Caiming Xiong, and Tong Zhang.
\newblock Rlhf workflow: From reward modeling to online rlhf.
\newblock \emph{arXiv preprint arXiv:2405.07863}, 2024.

\bibitem[Freund and Schapire(1999)]{freund1999adaptive}
Yoav Freund and Robert~E Schapire.
\newblock Adaptive game playing using multiplicative weights.
\newblock \emph{Games and Economic Behavior}, 29\penalty0 (1-2):\penalty0 79--103, 1999.

\bibitem[Gao et~al.(2023)Gao, Schulman, and Hilton]{gao2023scaling}
Leo Gao, John Schulman, and Jacob Hilton.
\newblock Scaling laws for reward model overoptimization.
\newblock In \emph{International Conference on Machine Learning}, pages 10835--10866. PMLR, 2023.

\bibitem[G{\"o}lz et~al.(2026)G{\"o}lz, Haghtalab, and Yang]{golz2026distortion}
Paul G{\"o}lz, Nika Haghtalab, and Kunhe Yang.
\newblock Distortion of ai alignment: Does preference optimization optimize for preferences?
\newblock \emph{Advances in Neural Information Processing Systems}, 38:\penalty0 25969--26007, 2026.

\bibitem[{Google}(2026)]{google2026gemini3}
{Google}.
\newblock {Gemini 3.1 Pro Model Card}.
\newblock \url{https://storage.googleapis.com/deepmind-media/Model-Cards/Gemini-3-1-Pro-Model-Card.pdf}, 2026.

\bibitem[Gui et~al.(2024)Gui, G{\^a}rbacea, and Veitch]{gui2024bonbon}
Lin Gui, Cristina G{\^a}rbacea, and Victor Veitch.
\newblock Bonbon alignment for large language models and the sweetness of best-of-n sampling.
\newblock \emph{Advances in Neural Information Processing Systems}, 37:\penalty0 2851--2885, 2024.

\bibitem[Halpern et~al.(2026)Halpern, Micha, Procaccia, and Shapira]{halpern2026pairwise}
Daniel Halpern, Evi Micha, Ariel Procaccia, and Itai Shapira.
\newblock Pairwise calibrated rewards for pluralistic alignment.
\newblock \emph{Advances in Neural Information Processing Systems}, 38:\penalty0 57882--57916, 2026.

\bibitem[Huang et~al.(2025)Huang, Block, Liu, Jiang, Krishnamurthy, and Foster]{huang2025best}
Audrey Huang, Adam Block, Qinghua Liu, Nan Jiang, Akshay Krishnamurthy, and Dylan~J Foster.
\newblock Is best-of-n the best of them? coverage, scaling, and optimality in inference-time alignment.
\newblock \emph{arXiv preprint arXiv:2503.21878}, 2025.

\bibitem[Jiang et~al.(2023)Jiang, Ren, and Lin]{jiang2023llm}
Dongfu Jiang, Xiang Ren, and Bill~Yuchen Lin.
\newblock Llm-blender: Ensembling large language models with pairwise ranking and generative fusion.
\newblock In \emph{Proceedings of the 61st Annual Meeting of the Association for Computational Linguistics (Volume 1: Long Papers)}, pages 14165--14178, 2023.

\bibitem[Khanov et~al.(2024)Khanov, Burapacheep, and Li]{khanov2024args}
Maxim Khanov, Jirayu Burapacheep, and Yixuan Li.
\newblock Args: Alignment as reward-guided search.
\newblock \emph{arXiv preprint arXiv:2402.01694}, 2024.

\bibitem[Maura-Rivero et~al.(2025)Maura-Rivero, Lanctot, Visin, and Larson]{maura2025jackpot}
Roberto-Rafael Maura-Rivero, Marc Lanctot, Francesco Visin, and Kate Larson.
\newblock Jackpot! alignment as a maximal lottery.
\newblock \emph{arXiv preprint arXiv:2501.19266}, 2025.

\bibitem[Munos et~al.(2024)Munos, Valko, Calandriello, Azar, Rowland, Guo, Tang, Geist, Mesnard, Fiegel, et~al.]{munos2024nash}
R{\'e}mi Munos, Michal Valko, Daniele Calandriello, Mohammad~Gheshlaghi Azar, Mark Rowland, Zhaohan~Daniel Guo, Yunhao Tang, Matthieu Geist, Thomas Mesnard, C{\^o}me Fiegel, et~al.
\newblock Nash learning from human feedback.
\newblock In \emph{Forty-first International Conference on Machine Learning}, 2024.

\bibitem[Nakano et~al.(2021)Nakano, Hilton, Balaji, Wu, Ouyang, Kim, Hesse, Jain, Kosaraju, Saunders, et~al.]{nakano2021webgpt}
Reiichiro Nakano, Jacob Hilton, Suchir Balaji, Jeff Wu, Long Ouyang, Christina Kim, Christopher Hesse, Shantanu Jain, Vineet Kosaraju, William Saunders, et~al.
\newblock Webgpt: Browser-assisted question-answering with human feedback.
\newblock \emph{arXiv preprint arXiv:2112.09332}, 2021.

\bibitem[Nemirovski(2004)]{nemirovski2004prox}
Arkadi Nemirovski.
\newblock Prox-method with rate of convergence o (1/t) for variational inequalities with lipschitz continuous monotone operators and smooth convex-concave saddle point problems.
\newblock \emph{SIAM Journal on Optimization}, 15\penalty0 (1):\penalty0 229--251, 2004.

\bibitem[Nesterov(2009)]{nesterov2009primal}
Yurii Nesterov.
\newblock Primal-dual subgradient methods for convex problems.
\newblock \emph{Mathematical programming}, 120\penalty0 (1):\penalty0 221--259, 2009.

\bibitem[{OpenAI}(2026)]{openai2026gpt5}
{OpenAI}.
\newblock {GPT-5.4 Thinking System Card}.
\newblock \url{https://deploymentsafety.openai.com/gpt-5-4-thinking/gpt-5-4-thinking.pdf}, 2026.

\bibitem[QwenTeam(2025)]{qwen3technicalreport}
QwenTeam.
\newblock Qwen3 technical report, 2025.
\newblock URL \url{https://arxiv.org/abs/2505.09388}.

\bibitem[Rafailov et~al.(2023)Rafailov, Sharma, Mitchell, Manning, Ermon, and Finn]{rafailov2023direct}
Rafael Rafailov, Archit Sharma, Eric Mitchell, Christopher~D Manning, Stefano Ermon, and Chelsea Finn.
\newblock Direct preference optimization: Your language model is secretly a reward model.
\newblock \emph{Advances in neural information processing systems}, 36:\penalty0 53728--53741, 2023.

\bibitem[Rakhlin and Sridharan(2013)]{rakhlin2013optimization}
Sasha Rakhlin and Karthik Sridharan.
\newblock Optimization, learning, and games with predictable sequences.
\newblock \emph{Advances in Neural Information Processing Systems}, 26, 2013.

\bibitem[Rosset et~al.(2024)Rosset, Cheng, Mitra, Santacroce, Awadallah, and Xie]{rosset2024direct}
Corby Rosset, Ching-An Cheng, Arindam Mitra, Michael Santacroce, Ahmed Awadallah, and Tengyang Xie.
\newblock Direct nash optimization: Teaching language models to self-improve with general preferences.
\newblock \emph{arXiv preprint arXiv:2404.03715}, 2024.

\bibitem[Schulman et~al.(2017)Schulman, Wolski, Dhariwal, Radford, and Klimov]{schulman2017proximal}
John Schulman, Filip Wolski, Prafulla Dhariwal, Alec Radford, and Oleg Klimov.
\newblock Proximal policy optimization algorithms.
\newblock \emph{arXiv preprint arXiv:1707.06347}, 2017.

\bibitem[Shao et~al.(2024)Shao, Wang, Zhu, Xu, Song, Bi, Zhang, Zhang, Li, Wu, et~al.]{shao2024deepseekmath}
Zhihong Shao, Peiyi Wang, Qihao Zhu, Runxin Xu, Junxiao Song, Xiao Bi, Haowei Zhang, Mingchuan Zhang, YK~Li, Yang Wu, et~al.
\newblock Deepseekmath: Pushing the limits of mathematical reasoning in open language models.
\newblock \emph{arXiv preprint arXiv:2402.03300}, 2024.

\bibitem[Shi et~al.(2024)Shi, Chen, Hu, Liu, Hajishirzi, Smith, and Du]{shi2024decoding}
Ruizhe Shi, Yifang Chen, Yushi Hu, Alisa Liu, Hannaneh Hajishirzi, Noah~A Smith, and Simon~S Du.
\newblock Decoding-time language model alignment with multiple objectives.
\newblock \emph{Advances in Neural Information Processing Systems}, 37:\penalty0 48875--48920, 2024.

\bibitem[Shirali et~al.(2026)Shirali, Nasr-Esfahany, Alomar, Mirtaheri, Abebe, and Procaccia]{shirali2026direct}
Ali Shirali, Arash Nasr-Esfahany, Abdullah Alomar, Parsa Mirtaheri, Rediet Abebe, and Ariel Procaccia.
\newblock Direct alignment with heterogeneous preferences.
\newblock \emph{Advances in Neural Information Processing Systems}, 38:\penalty0 69858--69898, 2026.

\bibitem[Skalse et~al.(2022)Skalse, Howe, Krasheninnikov, and Krueger]{skalse2022defining}
Joar Skalse, Nikolaus Howe, Dmitrii Krasheninnikov, and David Krueger.
\newblock Defining and characterizing reward gaming.
\newblock \emph{Advances in Neural Information Processing Systems}, 35:\penalty0 9460--9471, 2022.

\bibitem[Stiennon et~al.(2020)Stiennon, Ouyang, Wu, Ziegler, Lowe, Voss, Radford, Amodei, and Christiano]{stiennon2020learning}
Nisan Stiennon, Long Ouyang, Jeffrey Wu, Daniel Ziegler, Ryan Lowe, Chelsea Voss, Alec Radford, Dario Amodei, and Paul~F Christiano.
\newblock Learning to summarize with human feedback.
\newblock \emph{Advances in neural information processing systems}, 33:\penalty0 3008--3021, 2020.

\bibitem[Stroebl et~al.(2024)Stroebl, Kapoor, and Narayanan]{stroebl2024inference}
Benedikt Stroebl, Sayash Kapoor, and Arvind Narayanan.
\newblock Inference scaling flaws: The limits of llm resampling with imperfect verifiers.
\newblock \emph{arXiv preprint arXiv:2411.17501}, 3\penalty0 (8):\penalty0 14, 2024.

\bibitem[Swamy et~al.(2024)Swamy, Dann, Kidambi, Wu, and Agarwal]{swamy2024minimaximalist}
Gokul Swamy, Christoph Dann, Rahul Kidambi, Zhiwei~Steven Wu, and Alekh Agarwal.
\newblock A minimaximalist approach to reinforcement learning from human feedback.
\newblock \emph{arXiv preprint arXiv:2401.04056}, 2024.

\bibitem[Wang et~al.(2024{\natexlab{a}})Wang, Li, Chen, Cai, Zhu, Lin, Cao, Kong, Liu, Liu, et~al.]{wang2024large}
Peiyi Wang, Lei Li, Liang Chen, Zefan Cai, Dawei Zhu, Binghuai Lin, Yunbo Cao, Lingpeng Kong, Qi~Liu, Tianyu Liu, et~al.
\newblock Large language models are not fair evaluators.
\newblock In \emph{Proceedings of the 62nd Annual Meeting of the Association for Computational Linguistics (Volume 1: Long Papers)}, pages 9440--9450, 2024{\natexlab{a}}.

\bibitem[Wang et~al.(2024{\natexlab{b}})Wang, Bukharin, Delalleau, Egert, Shen, Zeng, Kuchaiev, and Dong]{wang2024helpsteer2}
Zhilin Wang, Alexander Bukharin, Olivier Delalleau, Daniel Egert, Gerald Shen, Jiaqi Zeng, Oleksii Kuchaiev, and Yi~Dong.
\newblock Helpsteer2-preference: Complementing ratings with preferences.
\newblock \emph{arXiv preprint arXiv:2410.01257}, 2024{\natexlab{b}}.

\bibitem[Wu et~al.(2024)Wu, Sun, Yuan, Ji, Yang, and Gu]{wu2024self}
Yue Wu, Zhiqing Sun, Huizhuo Yuan, Kaixuan Ji, Yiming Yang, and Quanquan Gu.
\newblock Self-play preference optimization for language model alignment.
\newblock \emph{arXiv preprint arXiv:2405.00675}, 2024.

\bibitem[Yao et~al.(2023)Yao, Yu, Zhao, Shafran, Griffiths, Cao, and Narasimhan]{yao2023tree}
Shunyu Yao, Dian Yu, Jeffrey Zhao, Izhak Shafran, Tom Griffiths, Yuan Cao, and Karthik Narasimhan.
\newblock Tree of thoughts: Deliberate problem solving with large language models.
\newblock \emph{Advances in neural information processing systems}, 36:\penalty0 11809--11822, 2023.

\bibitem[Ye et~al.(2024)Ye, Xiong, Zhang, Jiang, and Zhang]{ye2024theoretical}
Chenlu Ye, Wei Xiong, Yuheng Zhang, Nan Jiang, and Tong Zhang.
\newblock A theoretical analysis of nash learning from human feedback under general kl-regularized preference.
\newblock \emph{arXiv preprint arXiv:2402.07314}, 4\penalty0 (5):\penalty0 10, 2024.

\bibitem[Yu et~al.(2026)Yu, Wu, and Block]{yu2026curiosity}
Zhuohao Yu, Zhiwei~Steven Wu, and Adam Block.
\newblock From curiosity to caution: Mitigating reward hacking for best-of-n with pessimism.
\newblock \emph{arXiv preprint arXiv:2604.04648}, 2026.

\bibitem[Zhang et~al.(2023)Zhang, Bai, and Jiang]{zhang2023offline}
Yuheng Zhang, Yu~Bai, and Nan Jiang.
\newblock Offline learning in markov games with general function approximation.
\newblock In \emph{International Conference on Machine Learning}, pages 40804--40829. PMLR, 2023.

\bibitem[Zhang et~al.(2024)Zhang, Yu, Peng, Song, Tian, Huo, Jiang, Mi, and Yu]{zhang2024iterative}
Yuheng Zhang, Dian Yu, Baolin Peng, Linfeng Song, Ye~Tian, Mingyue Huo, Nan Jiang, Haitao Mi, and Dong Yu.
\newblock Iterative nash policy optimization: Aligning llms with general preferences via no-regret learning.
\newblock \emph{arXiv preprint arXiv:2407.00617}, 2024.

\bibitem[Zhang et~al.(2025)Zhang, Yu, Ge, Song, Zeng, Mi, Jiang, and Yu]{zhang2025improving}
Yuheng Zhang, Dian Yu, Tao Ge, Linfeng Song, Zhichen Zeng, Haitao Mi, Nan Jiang, and Dong Yu.
\newblock Improving llm general preference alignment via optimistic online mirror descent.
\newblock \emph{arXiv preprint arXiv:2502.16852}, 2025.

\bibitem[Zheng et~al.(2023)Zheng, Chiang, Sheng, Zhuang, Wu, Zhuang, Lin, Li, Li, Xing, et~al.]{zheng2023judging}
Lianmin Zheng, Wei-Lin Chiang, Ying Sheng, Siyuan Zhuang, Zhanghao Wu, Yonghao Zhuang, Zi~Lin, Zhuohan Li, Dacheng Li, Eric Xing, et~al.
\newblock Judging llm-as-a-judge with mt-bench and chatbot arena.
\newblock \emph{Advances in neural information processing systems}, 36:\penalty0 46595--46623, 2023.

\bibitem[Zhong et~al.(2022)Zhong, Xiong, Tan, Wang, Zhang, Wang, and Yang]{zhong2022pessimistic}
Han Zhong, Wei Xiong, Jiyuan Tan, Liwei Wang, Tong Zhang, Zhaoran Wang, and Zhuoran Yang.
\newblock Pessimistic minimax value iteration: Provably efficient equilibrium learning from offline datasets.
\newblock In \emph{International Conference on Machine Learning}, pages 27117--27142. PMLR, 2022.

\bibitem[Zhou et~al.(2025)Zhou, Fazel, and Du]{zhou2025extragradient}
Runlong Zhou, Maryam Fazel, and Simon~S Du.
\newblock Extragradient preference optimization (egpo): Beyond last-iterate convergence for nash learning from human feedback.
\newblock \emph{arXiv preprint arXiv:2503.08942}, 2025.

\end{thebibliography}

\newpage
\appendix

\section{Rejection Sampling}

In this section, we introduce the rejection sampling algorithm (\Cref{alg:rejection_sampling}) ~\citep{block2023sample, huang2025best}.
The algorithm draws $N$ samples i.i.d. from $\pi_{\reftext}(\cdot|x)$.
For each $y_i$, it samples a Bernoulli random variable $\xi_i$ where $\mathbb{P}(\xi_i = 1 \mid y_i) = \min \left\{ \frac{w(y_i \mid x)}{M}, 1 \right\}$.
The algorithm returns $y_i$ if the Bernoulli random variable $\xi_i = 1$.
If $\xi_i = 0$ for $\forall i \in [N]$, then the algorithm randomly samples a response from the reference policy $\pi_{\reftext}$.

The main purpose of the rejection sampling algorithm is to approximate a target policy $\pi(\cdot|x)$ from $\pi_{\reftext}$ by setting $w(\cdot|x) = \frac{\pi(\cdot|x)}{\pi_{\reftext}(\cdot|x)} $.
In the analysis of this paper, we mainly use the distribution induced by rejection sampling as an intermediate step.

\begin{algorithm}[H]
    \caption{Rejection Sampling ($\text{RejectionSampling}_{N, M}(w; \pi_{\text{ref}}, x)$)}
    \label{alg:rejection_sampling}
    \begin{algorithmic}[1]
    \State \textbf{Input}: Prompt $x$, base policy $\pi_{\text{ref}}$, importance weight $w$, truncation level $M$.
    
    \State Draw $ = (y_1, \dots, y_N, y_{N+1}) \sim \pi_{\text{ref}}(\cdot \mid x)$ i.i.d. 
    
    \For{$i = 1 \dots N$}
        \State Sample Bernoulli random variable $\xi_i$ such that $\mathbb{P}(\xi_i = 1 \mid y_i) = \min \left\{ \frac{w(y_i \mid x)}{M}, 1 \right\}$ 
        \If{$\xi_i = 1$}
            \State \textbf{Return} response $y = y_i$. 
        \EndIf
    \EndFor 
    \State \textbf{Return} response $y = y_{N+1}$. 
    \end{algorithmic}
\end{algorithm}

\section{Omitted Proofs from \Cref{sec:bon}}
\label{appendix:bon}

In this section, let $\mathcal{Y}_N = \{y_1, \dots, y_N\}$ denote
the candidate set drawn i.i.d.\ from $\pi_{\mathrm{ref}}(\cdot \mid x)$, let $\hat{\pi}$ be the policy returned by \Cref{alg:best_of_nash}, and let
$\tilde{\pi} = \delta_{y^*}$ with $y^* \in argmax_{y \in \mathcal{Y}}\mathbb{P}^*(y \succ \hat{\pi}|x)$ denote a pure best response to $\hat{\pi}$ under the true preference.
Since the maximum of the linear functional $\pi \mapsto \mathbb{P}^*(\pi \succ \hat{\pi}|x)$ is attained at a point mass, we have that $\DualGap(\hat{\pi}) = 2 \mathbb{P}^*(\tilde{\pi} \succ \hat{\pi}|x) - 1$. 
We define $\pi_R$ as the
distribution induced by
$\mathrm{RejectionSampling}_{N-1,M}\big(\tfrac{\tilde{\pi}}{\pi_{\mathrm{ref}}};
\pi_{\mathrm{ref}}, x\big)$ run on the same
candidate set $\mathcal{Y}_N$. 
Under this coupling,
every branch of \Cref{alg:rejection_sampling} returns an element of
$\mathcal{Y}_N$, so conditionally on $\mathcal{Y}_N$ we have
$\pi_R \mid \mathcal{Y}_N \in \Delta(\mathcal{Y}_N)$. 
We write $D_{\text{TV}}(\tilde{\pi}, \pi_R |\mathcal{Y}_N)$ for the TV distance between $\tilde{\pi}$ and $\pi_R$ given $\mathcal{Y}_N$.

\begin{lemma}
\label{lem:bon_pi_R_bound}
With $\hat{\pi}, \tilde{\pi}$, and $\pi_R$ as defined above, we have 
    \begin{equation*}
        \mathbb{P}^*(\pi_R \succ \hat{\pi}|x) \leq \frac{1}{2} + \varepsilon(x) \mathcal{C}_{\mathrm{uni}}(x)
    \end{equation*}
\end{lemma}

\begin{proof}
All displays below hold conditionally on $\mathcal{Y}_N$, and taking expectation over the draw of $\mathcal{Y}_N$ at the end yields the stated bound.
By the construction, $\{\widehat{\mathbb{P}}(y_i \succ y_j |x)\}_{i \in [N], j \in [N]}$ is a two-player constant-sum matrix game with domain on the dataset $\mathcal{Y}_{N} = \{ y_i\}_{i \in [N]} \subseteq \mathcal{Y}$. 
Since $(\hat{\pi}, \hat\pi)$ is a Nash Equilibrium computed by \Cref{alg:best_of_nash}, we have
\begin{equation*}
    (\hat{\pi}, \hat{\pi}) = \argmax_{\pi \in \Delta(\mathcal{Y}_{N})} \argmin_{\pi' \in \Delta(\mathcal{Y}_N)}\widehat{\mathbb{P}}(\pi \succ \pi').
\end{equation*}
Since $\pi_R \mid \mathcal{Y}_N \in \Delta(\mathcal{Y}_N)$, the definition of the Nash equilibrium gives
\begin{align*}
    &\widehat{\mathbb{P}}(\pi_R \succ \hat{\pi}|x) - \widehat{\mathbb{P}}(\hat{\pi} \succ \hat{\pi}|x) \\
    &= 
\mathbb{E}_{\mathcal{Y}_N}\mathbb{E}_{y \sim \pi_R|\mathcal{Y}_N, y' \sim \hat{\pi}|\mathcal{Y}_N,  y'' \sim \hat{\pi}|\mathcal{Y}_N}[\widehat{\mathbb{P}}(y \succ y''|x) - \widehat{\mathbb{P}}(y' \succ y''|x)] \\
&\leq 0.
\end{align*}
More concretely, given the column player's policy $\hat\pi$, the best policy that the row player takes should be $\hat\pi$.
On the other hand, conditioned on $\mathcal{Y}_N$, we calculate that 
\begin{align*}
    &|\mathbb{P}^*(\pi_R \succ \hat{\pi}|\mathcal{Y_N}) - \widehat{\mathbb{P}}(\pi_R \succ \hat{\pi}|\mathcal{Y_N})|\\
    &\leq \mathbb{E}_{y \sim \pi_R \mid \mathcal{Y}_N}\mathbb{E}_{y' \sim \hat{\pi} \mid \mathcal{Y}_N}|\mathbb{P}^*(y \succ y'|x) - \widehat{\mathbb{P}}(y \succ y'|x)| \\
    &\leq \mathbb{E}_{y \sim \pi_{\reftext}}\mathbb{E}_{y' \sim \pi_{\reftext}}[\frac{\pi_R(y \mid \mathcal{Y}_N)}{\pi_{\reftext}(y|x)}\frac{\hat{\pi}(y'|\mathcal{Y}_N)}{\pi_{\reftext}(y'|x)}|\mathbb{P}^*(y \succ y'|x) - \widehat{\mathbb{P}}(y \succ y'|x)|] \\
    &\leq  \sqrt{\mathbb{E}_{y \sim \pi_{\reftext}}[(\frac{\pi_R(y|\mathcal{Y}_N)}{\pi_{\reftext}(y|x)})^2] \mathbb{E}_{y' \sim \pi_{\reftext}}[(\frac{\hat{\pi}(y'|\mathcal{Y}_N)}{\pi_{\reftext}(y'|x)})^2]} \sqrt{ \mathbb{E}_{y \sim \pi_{\reftext}}\mathbb{E}_{y' \sim \pi_{\reftext}}[|\mathbb{P}^*(y \succ y'|x) - \widehat{\mathbb{P}}(y \succ y'|x)|^2]} \\
    &\leq \sqrt{\mathbb{E}_{y \sim \pi_{R}}[\frac{\pi_R(y|\mathcal{Y}_N)}{\pi_{\reftext}(y|x)}]} \sqrt{\mathbb{E}_{y' \sim \hat{\pi}}[\frac{\hat{\pi}(y'|\mathcal{Y}_N)}{\pi_{\reftext}(y'|x)}]}  \cdot\varepsilon(x) \\
    &\leq \mathcal{C}_{\text{uni}}(x)\varepsilon(x).
\end{align*}


    Combining the above computation, we have
    \begin{align*}
        \mathbb{P}^*(\pi_R \succ \hat{\pi}|x) &\leq \mathbb{E}_{\mathcal{Y}_N}[\widehat{\mathbb{P}}(\pi_R \succ \hat{\pi} \mid \mathcal{Y}_N) + |\mathbb{P}^*(\pi_R \succ \hat{\pi} \mid \mathcal{Y}_N) - \widehat{\mathbb{P}}(\pi_R \succ \hat{\pi} \mid \mathcal{Y}_N)|] \\
        &\leq  \mathbb{E}_{\mathcal{Y}_N}[\widehat{\mathbb{P}}(\hat{\pi} \succ \hat{\pi} \mid \mathcal{Y}_N) + |\mathbb{P}^*(\pi_R \succ \hat{\pi} \mid \mathcal{Y}_N) - \widehat{\mathbb{P}}(\pi_R \succ \hat{\pi} \mid \mathcal{Y}_N)| ]\\
        &\leq \frac{1}{2} + \mathbb{E}_{\mathcal{Y}_N}[|\mathbb{P}^*(\pi_R \succ \hat{\pi}|x) - \widehat{\mathbb{P}}(\pi_R \succ \hat{\pi}|x)|] \\
        &\leq \frac{1}{2} + \varepsilon(x) \mathcal{C}_{\text{uni}}(x),
    \end{align*}
and the proof is complete.

\end{proof}
\begin{lemma}
\label{lem:total_variance}
   With $\tilde{\pi}$ and $\pi_R$ as defined above with $M = \frac{N-1}{\log(4/\varepsilon^2(x))}$, and any $\sigma(\mathcal{Y}_N)$-measurable policy $\hat{\pi} \in \Delta(\mathcal{Y}_N)$, we have that
    \begin{equation*}
        D_{\text{TV}}(\tilde{\pi}, \pi_R)  \leq \frac{1}{4}\mathcal{C}_{\mathrm{uni}}(x)\varepsilon^2(x)
    \end{equation*} 
    when $N \geq 4\log(\frac{2}{\varepsilon(x)}) \cdot \mathcal{C}_{\uni}(x)$.
    The statement holds for the output of \Cref{alg:best_of_nash} and \Cref{alg:nmd}.
\end{lemma}

\begin{proof}
We omit the dependence on $x$.
First, the condition on $N$ implies $N-1 \geq 2\mathcal{C}_{\text{uni}}\log(2/\varepsilon) = \mathcal{C}_{\text{uni}}\log(4/\varepsilon^2)$, i.e. $M \geq \mathcal{C}_{
\text{uni}} \geq 1/\pi_{\text{ref}}(y)$ for any $y \in \mathcal{Y}$; hence $w(y)/M \leq 1$ for all $y$ and the acceptance probability in \Cref{alg:rejection_sampling} is exactly $w(y)/M$.
Since $w = \tilde{\pi}/\pi_{\text{ref}}$ is supported on the single point $y^*$, only candidates equal to $y^*$ can be accepted.
Therefore, conditioned on $\mathcal{Y}_N$, we have 
\begin{equation*}
    \pi_R|\mathcal{Y}_N = q \delta_{y_N} + (1-q) \delta_{y^*}, \ q = (1-\frac{1}{M\pi_{\text{ref}}(y^*)})^{m(y^*)}, 
\end{equation*}
where $m(y) = \#\{ i \leq N-1:y_i = y\}$, and consequently $D_{\text{TV}}(\tilde{\pi}, \pi_R | \mathcal{Y}_N) \leq q$.

Then we bound $q$:
\begin{equation*}
    q \leq \sum_{y \in \mathcal{Y}}(1-\frac{1}{M\pi_{\text{ref}}(y)})^{m(y)}.
\end{equation*}
For any fixed $y$, $m(y) \sim \text{Bin}(N-1, \pi_{\text{ref}}(y))$, and the probability generating function of the binomial gives
\begin{equation*}
    \mathbb{E}[(1- \frac{1}{M\pi_{\text{ref}}(y)})^{m(y)}] = (1- \pi_{\text{ref}}(y) \cdot \frac{1}{M\pi_{\text{ref}}(y)})^{N-1} = (1-\frac{1}{M})^{N-1} \leq e^{-(N-1)/M}.
\end{equation*}

By the definition of $\mathcal{C}_{\text{uni}}(x)$ we have $\mathcal{C}_{\text{uni}}(x) \geq \frac{1}{\pi_{\text{ref}}(y)}$ for any $y$, which along with $\sum_{y}\pi_{\text{ref}}(y) = 1$ implies
$|\mathcal{Y}| \leq \mathcal{C}_{\text{uni}}$.
By substituting $M = \frac{N-1}{\log(4/\varepsilon^2)}$, we bound
\begin{align*}
    D_{\text{TV}}(\tilde{\pi}, \pi_R) & \leq \mathbb{E}[q] \\
    &\leq \mathcal{C}_{\text{uni}} e^{-(N-1)/M} \\
    &\leq \mathcal{C}_{\text{uni}} \cdot \frac{\varepsilon^2}{4}.
\end{align*}
Then we take the expectation on $\mathcal{Y}_N$ and derive the desired result.
\end{proof}


\section{Omitted Proofs from \Cref{sec:nmd}}
\label{appendix:nmd}
The following technical lemma takes the result from optimistic mirror descent literature~\citep{rakhlin2013optimization} that studies the general Bregman divergence.
The KL divergence can be written by the Bregman divergence property:
\begin{equation*}
    \text{KL}(\pi \| \pi') = D_\psi(\pi, \pi') = \psi(\pi) - \psi(\pi') - \langle \nabla \psi(\pi'), \pi - \pi' \rangle,
\end{equation*}
where $\psi(\pi) = \sum_y \pi(y) \log \pi(y)$.
\begin{lemma}[Corollary of Lemma 1 in \citet{rakhlin2013optimization}]
\label{lem:omd_technical}
    Denote $\pi_t$ and $\hat{r}_t$ as is computed in \Cref{alg:nmd}.
    For any $\pi' \in \Delta(\mathcal{Y}_N)$, we have that 
    \begin{equation*}
        \sum_{t=1}^T \langle \pi' - \pi_t, \hat{r}_t \rangle \le \beta \cdot\text{KL}(\pi' \| \pi'_1) + \frac{1}{\beta} \sum_{t=1}^T \|\hat{r}_t - \hat{r}_{t-1}\|_\infty^2 - \frac{\beta}{4} \sum_{t=2}^T \|\pi_t - \pi_{t-1}\|_1^2.
    \end{equation*}
\end{lemma}
In the following, we provide a technical lemma that quantifies the samples required to compute an approximate Nash policy under the zero-sum matrix game $\widehat{\mathbb{P}}$.
\begin{lemma}
\label{lem:omd_hat_oracle}
Denote $\hat{\pi} = \frac{1}{T}\sum_{t = 1}^{T}\pi_t$ as the policy returned by \Cref{alg:nmd}.
    Given data size $N$ and a sufficiently small value $\epsilon >0$, we set $\beta = 2$ and $T = \lceil (2\log N + 1/2)/\epsilon \rceil$.
    Then for any $\pi' \in \Delta(\mathcal{Y}_N)$, we have 
    \begin{equation*}
        \widehat{\mathbb{P}}(\pi' \succ \hat\pi)  \leq \frac{1}{2} + \epsilon.
    \end{equation*}
\end{lemma}
\begin{proof}
Define $\hat{r}_0 := 0$.
Recall that $\hat\pi = \frac{1}{T}\sum_{t = 1}^T\pi_t$.
For any policy $\pi' \in \Delta(\mathcal{Y}_N)$, we have that
\begin{align*}
    \widehat{\mathbb{P}}(\pi' \succ \hat{\pi}) 
    &= \frac{1}{T}\sum_{t = 1}^{T}\widehat{\mathbb{P}}(\pi' \succ \pi_t) \\
    &\leq \frac{1}{2} + \frac{1}{T}\sum_{t = 1}^{T}[\widehat{\mathbb{P}}(\pi' \succ \pi_t) - \widehat{\mathbb{P}}(\pi_t \succ \pi_t)]   \\
    &\leq \frac{1}{2} + \frac{1}{T}\sum_{t=1}^T \langle \pi' - \pi_t, \hat{r}_t \rangle.
\end{align*}

Note that for any $\pi' \in \Delta(\mathcal{Y}_N)$ and $\pi'_1$ is a uniform distribution on $\mathcal{Y}_N$, we have 
\begin{equation*}
    \text{KL}(\pi' \| \pi'_1) \leq \log(N).
\end{equation*}
For $t \ge 2$, we have $|\hat{r}_t(y) - \hat{r}_{t-1}(y)| = |\sum_{y'} \hat{P}(y \succ y')(\pi_t(y') - \pi_{t-1}(y'))| \le \Vert\pi_t - \pi_{t-1}\Vert_1$. Substituting into \Cref{lem:omd_technical}, for any $\beta \geq 2$:

$$\sum_{t=1}^{T} \langle \pi' - \pi_t, \hat{r}_t \rangle \le \beta \log N + \tfrac{1}{\beta}\Vert\hat{r}_1\Vert_{\infty}^2 + \Big(\tfrac{1}{\beta} - \tfrac{\beta}{4}\Big) \sum_{t=2}^{T} \Vert\pi_t - \pi_{t-1}\Vert_1^2 \le \beta \log N + \tfrac{1}{\beta}$$

where the last inequality holds because $\Vert\hat{r}_1\Vert_{\infty} \le 1$ and, for any $\beta \ge 2$, $\tfrac{1}{\beta} - \tfrac{\beta}{4} \le 0$.
Setting $\beta = 2$ and $T = \lceil (2\log N + 1/2)/\epsilon \rceil$ yields

$$\hat{P}(\pi' \succ \hat{\pi}) \le \tfrac{1}{2} + \tfrac{1}{T}\sum_{t=1}^{T} \langle \pi' - \pi_t, \hat{r}_t \rangle \le \tfrac{1}{2} + \epsilon$$

for all $\pi' \in \Delta(\mathcal{Y}_N)$.

\end{proof}

With the above technical lemma, we can finally show the duality gap bound for NMD.

\thmnmd*

\begin{proof}
First fix data size $N$.
For any $\pi \in \Delta(\mathcal{Y}_N)$, we have that
\begin{align*}
    \mathbb{P}^*(\pi \succ \hat{\pi}) 
    &\leq  \widehat{\mathbb{P}}(\pi \succ \hat{\pi}) + |\widehat{\mathbb{P}}(\pi \succ \hat{\pi}) - \mathbb{P}^*(\pi \succ \hat{\pi})|  \\   
    &\leq \widehat{\mathbb{P}}(\pi \succ \hat{\pi}) + \mathcal{C}_{\text{uni}}(x)\varepsilon(x) \\
    &\leq \frac{1}{2} + 2\mathcal{C}_{\text{uni}}(x)\varepsilon(x).
\end{align*}
Here the third line follows from \Cref{lem:omd_hat_oracle}, and the second line is the same computation as in \Cref{lem:bon_pi_R_bound} so we omit the details.

Let $\tilde{\pi} = \delta_{y^{*}}$ be a pure best response to $\hat{\pi}$ and
$\pi_R$ the coupled rejection-sampling law, both as defined at the head of
Appendix~B.
Conditioned on $\mathcal{Y}_N$,
\begin{align*}
  \mathbb{P}^*(\tilde{\pi} \succ \hat{\pi} | \mathcal{Y}_N)
&\le
\mathbb{P}^*(\pi_R \succ \hat{\pi} | \mathcal{Y}_N) + |\mathbb{P}^*(\pi_R \succ \hat{\pi} | \mathcal{Y}_N) - \mathbb{P}^*(\tilde{\pi} \succ \hat{\pi} | \mathcal{Y}_N)|\\
&\leq \mathbb{P}^*(\pi_R \succ \hat{\pi} \mid \mathcal{Y}_N)
+ 2\,D_{\text{TV}}(\tilde{\pi}, \pi_R | \mathcal{Y}_N) \\
&\le \frac{1}{2} + 2\mathcal{C}_{\text{uni}}(x)\varepsilon(x)
+ 2\,D_{\text{TV}}(\tilde{\pi}, \pi_R | \mathcal{Y}_N),  
\end{align*}
where the last inequality applies the first display of this proof to the
policy $\pi_R \mid \mathcal{Y}_N \in \Delta(\mathcal{Y}_N)$. Taking expectation over $\mathcal{Y}_N$ and
applying \Cref{lem:total_variance} with $M = \frac{N-1}{\log(4/\varepsilon^2(x))}$,
\[
\mathbb{P}^*(\tilde{\pi} \succ \hat{\pi} \mid \mathcal{Y}_N )
\;\le\; \tfrac{1}{2} + 2\mathcal{C}_{\text{uni}}(x)\varepsilon(x)
+ \tfrac{1}{2}\mathcal{C}_{\text{uni}}(x)\varepsilon^2(x) 
\;\leq\; \tfrac{1}{2} + \tfrac{5}{2}\mathcal{C}_{\text{uni}}(x)\varepsilon(x),
\]
and hence
$\DualGap(\hat{\pi})
= 2\,\mathbb{E}[\mathbb{P}^*(\tilde{\pi} \succ \hat{\pi} \mid \mathcal{Y}_N)] - 1
\le  5\varepsilon(x)\mathcal{C}_{\text{uni}}(x)$. 
\end{proof}

\section{Omitted Proofs from \Cref{sec:lower}}
        
    

        
    

\label{appendix:lower}
\thmlower*

\begin{proof}
     Set $\delta_0 := \varepsilon_0 K / (2\sqrt{2}) \leq 1/6$ by the assumption on $\varepsilon_0$. 
     Define the imperfect oracle:
\begin{equation*}
\widehat{ \mathbb{P}}(y_1 \succ y_j) = \tfrac{1}{2} + \delta_0 \quad \text{for all } K \geq j \geq 2, \qquad \widehat {\mathbb{P}}(y_j \succ y_k) = \tfrac{1}{2} \quad \text{for all } K \geq j,k \geq 2.    
\end{equation*} 
Response $y_1$ is the unique dominant strategy by $\widehat{\mathbb{P}}$. 
Now write
 \begin{equation*}
     \hat\pi(y_1) = p, \hat\pi(y_2) = q,
 \end{equation*}
for some $p,q \in [0,1]$.

In the following, we construct two true preference oracles in two worlds $A$ and $B$.
In the world $A$, $\mathbb{P}^*_{A}$ agrees with $\widehat{\mathbb{P}}$ except 
\begin{equation*}
    \mathbb{P}^*_A(y_1 \succ y_2) = \tfrac{1}{2} + 3\delta_0, \qquad \mathbb{P}^*_A(y_2 \succ y_1) = \tfrac{1}{2} - 3\delta_0.
\end{equation*}
In the world $B$, $\mathbb{P}^*_{B}$ agrees with $\widehat{\mathbb{P}}$ except 
\begin{equation*}
    \mathbb{P}^*_B(y_1 \succ y_2) = \tfrac{1}{2} - \delta_0, \qquad \mathbb{P}^*_B(y_2 \succ y_1) = \tfrac{1}{2} + \delta_0.
\end{equation*}
Since $\delta_0 \leq 1/6$, both $P^*_A$ and $P^*_B$ take values in $[0,1]$.

Now we compute the oracle quality $\varepsilon_A^2(x)$ and $\varepsilon_B^2(x)$:
\begin{equation*}
    \varepsilon^2_A(x) = \varepsilon^2_B(x) = \frac{8\delta_0^2}{K^2} = \varepsilon_0^2. 
\end{equation*}
In World A, $y_1$ remains strictly dominant, so we compute:
\begin{equation*}
    \mathbb{P}^*_A(y_1 \succ \hat\pi)  = \frac{p}{2} + q(\frac{1}{2} + 3 \delta_0) + (1-p-q)(\frac{1}{2} + \delta_0) \geq \frac{1}{2} + (1-p)\delta_0.
\end{equation*}
Therefore, we lower bound the duality gap in world A: 
\begin{equation*}
    \DualGap_A(\hat\pi) 
     \geq 2\mathbb{P}^*_A(y_1 \succ \hat{\pi}) - 1 \geq 2(1-p)\delta_0.
\end{equation*}
  In World B, $y_2$ beats $y_1$ by margin $\delta_0$, while $y_2$ ties with $y_j$ for $j \geq 3$, thus $y_2$ is dominant.
  So we compute
  \begin{equation*}
      \mathbb{P}^*_B(y_2 \succ \hat\pi) = p(\frac{1}{2} + \delta_0) + \frac{q}{2} + (1-p-q)\cdot \frac{1}{2} = \frac{1}{2} + p\delta_0.
  \end{equation*}
  Therefore, we lower bound the duality gap in world B:
  \begin{equation*}
 \DualGap_B(\hat\pi) \geq 2\mathbb{P}^*_B(y_2 \succ \hat{\pi}) - 1 \geq  2p\delta_0.
  \end{equation*}
  The adversary selects the world that is worse for the algorithm:
  \begin{equation*}
      \DualGap(\hat\pi) \geq \max\bigl(2(1-p)\delta_0,\; 2p\delta_0\bigr).
  \end{equation*}
  The algorithm minimizes this by setting $p = 1/2$, yielding:
  \begin{equation*}
      \DualGap(\hat\pi) \geq \delta_0 = \frac{\varepsilon_0 K}{2\sqrt{2}} = \frac{\varepsilon_0\, \mathcal{C}_{\uni}}{2\sqrt{2}}
  \end{equation*}
  since $\mathcal{C}_{\uni}(x) = K$.
 \end{proof}

 \section{Additional Empirical Results}
\label{sec:add_exp}

\subsection{Further Experimental Details}\label{app:details}

All win/draw/lose judgments are produced by DeepSeek-V4-Flash (temperature = 0, at most $4$ output tokens). 
The verbatim prompt given to the judge:
\begin{quote}
\small\ttfamily
\textbf{System:} You are an impartial expert judge evaluating the quality of
two AI assistant responses to the same user prompt. Judge which response better
follows the user's instructions and is more helpful, correct, coherent and
appropriately detailed for the request. Do not let the length of a response,
the order in which the responses are presented, or stylistic flourishes bias
your decision. Output exactly one character: 'A' if Response A is better, or
'B' if Response B is better. Do not output anything else.

\medskip
\textbf{User:}\\
{[}User Prompt{]}\\
\{prompt\}\\[4pt]
{[}Response A{]}\\
\{a\}\\[4pt]
{[}Response B{]}\\
\{b\}\\[4pt]
Which response is better? Answer with a single letter: A or B.
\end{quote}

\subsection{Additional Experiments}
In this section, We compare our methods against three baselines: Borda Best-of-$N$, standard reward-based Best-of-$N$, and a fine-tuned Nash-MD-PG model~\citep{munos2024nash}.

\begin{algorithm}[H]
    \caption{Borda Best-of-N}
    \label{alg:borda_bon}
    \begin{algorithmic}[1]
        \State \textbf{Input}: Prompt $x$, reference policy $\pi_{\reftext}$, preference oracle $\widehat{\mathbb{P}}$, sample size $N$.
        \State Draw $\widehat{\mathcal{Y}}_N = (y_1, \dots, y_N) \sim \pi_{\reftext}(\cdot|x)$ i.i.d.
        \State Query $\widehat{P}_{ij} \gets \widehat{\mathbb{P}}(y_i \succ y_j \mid x)$ for all $1 \leq i < j \leq N$, and set $\widehat{P}_{ji} \gets 1 - \widehat{P}_{ij}$, $\widehat{P}_{ii} \gets \tfrac{1}{2}$.
        \State Compute the Borda score for each response:
        \begin{equation*}
            \hat{r}(y_i) \gets \frac{1}{N-1} \sum_{j \neq i} \widehat{P}_{ij}, \quad \forall\, i \in [N].
        \end{equation*}
        \State \textbf{Return} $\widehat{y} \gets \arg\max_{y \in \widehat{\mathcal{Y}}_N} \hat{r}(y)$.
    \end{algorithmic}
\end{algorithm}

\paragraph{A New Borda Best-of-$N$ Baseline.}
To control the oracle strength, we propose a Borda Best-of-$N$ method under the same preference oracle, shown in \Cref{alg:borda_bon}.
We report the result in \Cref{tab:bon-borda}.
We find that Borda Best-of-N attains win-rates close to Best-of-Nash and NMD across all three datasets.
To understand this, we ran a diagnostic on the dataset: in $73.8\%$ of $N=64$ sub-samples, the Borda winner is a \emph{Condorcet winner} of the empirical preference matrix---it beats every other candidate pairwise---in which case the Nash equilibrium of the sub-game is exactly the pure strategy on that response, so all three methods return the same response. 
The agreement is thus a structural property of the data rather than evidence that the equilibrium computation is redundant. 
On the remaining non-Condorcet sub-samples, argmax-style rules carry no guarantee.

\begin{table}[t]
\centering
\small
\setlength{\tabcolsep}{6pt}
\renewcommand{\arraystretch}{1.1}
\begin{tabular}{lcccc}
\toprule
Dataset & Base SFT  & Best-of-Nash  & Borda Best-of-N \\
\midrule
TLDR      & 62.9\%  & 73.5\%  & 72.5\% \\
HelpSteer2  & 44.1\%  & 68.8\%  & 67.9\%  \\
UltraFeedback & 23.9\%  & 48.8\%  & 45.7 \% \\
\bottomrule
\end{tabular}
\caption{Comparison of expected win-rate (EWR) across three datasets.
Best-of-Nash and Borda Best-of-N sample $N = 64$ responses from LLaMA3-SFT and we use LLaMA3-PM as preference model.}
\label{tab:bon-borda}
\end{table}

\paragraph{Standard Best-of-N.}
We implement the standard Best-of-N method under a fine-tuned
Bradley--Terry reward model. 
We train the reward model and a pairwise preference model under an identical protocol,
differing \emph{only} in architecture: both fine-tune the same backbone
(\href{https://huggingface.co/Qwen/Qwen3-4B-Instruct-2507}{Qwen3-4B-Instruct-2507}~\citep{qwen3technicalreport}) with a six-criterion linear head on the same
HelpSteer2 training pairs.
we draw $N = 128$ candidates per prompt from the
base policy (LLaMA3-SFT, temperature $1.0$) and compare three selectors: base SFT, Best-of-N under the Bradley-Terry reward, and Best-of-Nash under the pairwise
preference matrix.
We report the expected win rate (EWR) against
the prompt's human-preferred response, following the protocol
of \Cref{sec:exp}.
\Cref{tab:bon-comparison} reports the results: both methods improve over the base policy by $17$--$21$ percent,
confirming that a well-trained oracle of either form provides a strong
selection signal.
Best-of-Nash attains a $4.0\%$ higher win rate than
reward-based Best-of-N.

\begin{table}[h]
\centering
\begin{tabular}{lc}
\toprule
Selector & Expected win rate \\
\midrule
Base policy (random candidate) & 47.0\% \\
Best-of-N (Bradley--Terry reward) & 64.0\% \\
Best-of-Nash (pairwise preference) & \textbf{68.0\%} \\
\bottomrule
\end{tabular}
\caption{Comparison of Best-of-N and Best-of-Nash on held-out HelpSteer 2 prompts.
}
\label{tab:bon-comparison}
\end{table}

\paragraph{A Fine-tuned Nash-MD Baseline.}
Finally, we compare our methods against fine-tuning with general preferences: we fine-tune the base policy with Nash-MD-PG method~\citep{munos2024nash} on HelpSteer2 and evaluate it against Best-of-Nash and NMD applied to the same base policy at inference time.
We use Qwen3-0.6B as the base policy and LLaMA3-PM as the preference model.
For the Nash-MD-PG method, we train the total of $1$ epoch with learning rate of 2e-6, KL regularization coefficient of 0.01, a mixture coefficient of $0.5$, and temperature $0.7$.
We report the result in \Cref{tab:nash-md}.
Fine-tuned Nash-MD-PG improves over the base
policy ($40.5\%$ vs.\ $34.2\%$ EWR), while Best-of-Nash and NMD---using the same
preference model and no parameter updates---improve substantially
further ($50.5\%$ and $52.0\%$).
Notably, Nash-MD's gain comes largely
from converting losses into draws (a $59.0\%$ draw rate vs.\ $42.3\%$
for the base) while its outright win rate does not increase, whereas Best-of-Nash
and NMD raise the win rate itself. 
At this model scale and training
budget, inference-time equilibrium computation thus extracts more from
the same preference model than fine-tuning on it.
Thus, we read
this as evidence that our methods are a strong training-free
alternative.

\begin{table}[t]
\centering
\small
\setlength{\tabcolsep}{6pt}
\renewcommand{\arraystretch}{1.1}
\begin{tabular}{lccccc}
\toprule
Method & Win  & Draw & Lose  & EWR \\
\midrule
Base policy      & 13.0\%  & 42.3\%  & 44.7\% & 34.2\% \\
Nash-MD (fine-tuned) & 14.0\%  & 55.0\%  & 31.0\% & 41.5\%  \\ 
Best-of-Nash ($N = 64$)  & 29.0\%  & 43.0\% & 28.0\% & 50.5\%  \\
NMD ($N=64, \beta = 1$) & 33.0\%  & 38.0\%  & 29.0\% & 52.0\%  \\
\bottomrule
\end{tabular}
\caption{Comparison of Best-of-Nash, NMD, and a fine-tuned Nash-MD on HelpSteer2.
Best-of-Nash and NMD sample $N = 64$ responses from Qwen3-0.6B and we use LLaMA3-PM as preference model.}
\label{tab:nash-md}
\end{table}

\newpage

\end{document}